\documentclass{article}
\PassOptionsToPackage{numbers, compress}{natbib}
\usepackage[preprint]{neurips_2026}
\usepackage[utf8]{inputenc} % allow utf-8 input
\usepackage[T1]{fontenc}    % use 8-bit T1 fonts
\usepackage{amsmath}        % for math equations
\usepackage{amssymb}
\usepackage{hyperref}       % hyperlinks
\usepackage{url}            % simple URL 
\usepackage{amsthm}
\newtheorem{assumption}{Assumption}
\newtheorem{definition}{Definition}
\newtheorem{proposition}{Proposition}
\newtheorem{corollary}{Corollary}

\usepackage{booktabs}       % professional-quality tables
\usepackage{cleveref}
\usepackage{amsfonts}       % blackboard math symbols
\usepackage{nicefrac}       % compact symbols for 1/2, etc.
\usepackage{microtype}      % microtypography
\usepackage[table]{xcolor}  % colors and table row colors
\usepackage{graphicx}       % for figures
\usepackage{multirow} 
\usepackage{placeins}       % keep appendix floats within their section

\newtheorem{lemma}{Lemma}

\title{Mitigating Bias in Low-SNR Financial Reinforcement Learning via Quantum Representations}
\author{
\normalfont
Zeyu LIU\thanks{Equal contribution.}\quad
Xuanzhi FENG\footnotemark[1]\quad
LAI Sing Kwong\footnotemark[1]\quad
Yuanchen GAO\quad
Xiaoyi PANG\\
Hualei ZHANG\quad
Jingcai GUO\quad
Jie ZHANG\thanks{Corresponding authors.}\quad
Song GUO\footnotemark[2]\\
The Hong Kong University of Science and Technology\\
\texttt{zliugj@connect.ust.hk}\quad
\texttt{csejzhang@connect.ust.hk}\quad
\texttt{songguo@ust.hk}
}
\begin{document}

\maketitle
\begin{abstract}
The financial market is a typical low signal-to-noise ratio (SNR) setting, which often destabilizes off-policy maximum-entropy methods like Soft Actor-Critic (SAC). Specifically, noisy state representations may produce unreliable Q-value estimates, and bootstrapping amplifies these errors, forming a failure mode we call the ``Financial Entropy Trap''. In this paper, we propose \texttt{FPQC-SAC}, an efficient and plug-and-play SAC variant that places a compact and bounded Parameterized Quantum Circuit (PQC) before the actor and critic networks to constrain feature propagation at the representation level, rather than filtering raw inputs or regularizing Q-values after bootstrapping. Notably, \texttt{FPQC-SAC} reduces the impact of extreme market fluctuations on Bellman target estimation, while trainable quantum entanglement preserves flexible cross-asset interactions. Empirical evaluations on real-world portfolio management tasks demonstrate that FPQC-SAC substantially enhances out-of-sample stability and cumulative returns by achieving a 66.89\% relative gain in cumulative return over standard unconstrained SAC and outperforms the best continuous-control deep reinforcement learning baseline by $\sim$27\%. open-source code is available at \url{https://github.com/ZeyuLIU-UST/FPQC-SAC-main}.
\end{abstract}

% ------------------------------------------------------------------------
% 1. Introduction
% ------------------------------------------------------------------------

\section{Introduction}

\textbf{Financial markets are inherently low signal-to-noise ratio (SNR) environments.} Useful financial signals are often extremely weak, whereas short-term fluctuations, market microstructure noise, investor sentiment, and liquidity shocks are dominant. Prior studies have shown that financial prediction naturally suffers from these high noises, along with non-stationarity and distribution shifts. This makes models prone to fitting historical noise patterns, consequently suffering severe performance degradation in out-of-sample testing or real trading environments \cite{gu2020empirical, bao2017deep, bailey2017probability, zhang2025major}. Under such extreme uncertainty, sequential financial decision-making faces profound challenges.

At the core of this challenge lies sequential asset allocation: continuously determining when and how much to buy, hold, or sell. While traditional statistical models struggle to characterize complex market-state transitions, Deep Reinforcement Learning (DRL) provides a promising end-to-end framework \cite{liu2020finrl, liu2022finrlmeta, jiang2017deep, zhang2019deep, pippas2024evolution}. Among these, \textbf{maximum-entropy reinforcement learning}, represented by Soft Actor-Critic (SAC) \cite{haarnoja2018soft}, explicitly maximizes policy entropy to encourage stochastic exploration, aiming to maintain trading-policy diversity and generate high-quality decisions.

However, the Actor-Critic architecture becomes significantly fragile in noise-dominant markets. Entropy-driven stochastic exploration does not necessarily gather effective information; instead, it may amplify noise-dominated update directions. This prevents the policy from stably capturing weak predictive signals. Consequently, noise contaminates value estimation and guides erroneous exploration, forming a self-reinforcing loop. We refer to this critical failure mode of maximum-entropy RL in financial decision-making as the \textbf{\textit{Financial Entropy Trap}}.

To escape the Financial Entropy Trap, we first examine existing stabilization mechanisms. We consider two direct solutions: transferring general RL stabilization methods, and introducing explicit input filtering. Unfortunately, both paradigms exhibit inherent limitations:
\begin{itemize}
    \item \textbf{Output-side regularization intervenes too late.} Methods like clipped double-Q learning \cite{fujimoto2018addressing}, large-scale critic ensembles \cite{chen2021randomized}, and distributional critic truncation \cite{kuznetsov2020controlling} mainly act at the end of the update pipeline. Even contemporary advancements utilizing dropout Q-functions \cite{DroQ2022} or cross-Q updates \cite{CrossQ2024} focus primarily on stabilizing network-layer gradients during Bellman bootstrapping. Once market noise has passed through conventional encoders and contaminated intermediate representations \cite{lyle2024normalization, tran2026hybrid}, posterior corrections can only provide \textit{after-the-fact} mitigation. It becomes nearly impossible to recover weak predictive signals that have already been mixed or distorted.
    \item \textbf{Input-side filtering intervenes too rigidly.} Classical input-denoising pipelines, such as Kalman filtering, wavelet denoising, and Fourier-domain filtering, attempt to directly suppress noise before states enter the model \cite{abdulkadir2013unscented, alrumaih2002wavelet, song2021forecasting}. However, these methods usually rely on strong handcrafted priors, e.g., treating noise as measurement uncertainty or approximating it as high-frequency perturbations. Because useful financial signals themselves may also be transient and high-frequency, aggressive filtering often erases scarce valuable signals alongside the noise.
\end{itemize}
Neither approach fundamentally blocks the propagation of noise perturbations into value estimation while preserving weak financial signals.

To bridge this gap, we draw inspiration from a structurally relevant physical phenomenon in quantum optics \cite{kong2013cancellation}, where destructive interference partially cancels noise fluctuations without degrading the primary signal. This provides a crucial insight for low-SNR RL: effective noise suppression should not force a rigid choice between isolated pre-filtering or post-clipping, but rather introduce an incremental, structural constraint during signal representation. Inspired by this, we propose the \textbf{FPQC-SAC} architecture by embedding a Parameterized Quantum Circuit (PQC) at the representation front-end of SAC. As a plug-and-play enhancement, PQC goes beyond simply adding a new module; it leverages the natural interference structure induced by entangled evolution and measurement readout. This allows unstructured noise to statistically cancel within a high-dimensional representation space, effectively decoupling noise suppression from signal retention and overcoming the limitations of classical approaches.

The main contributions of this paper are summarized as follows:
\begin{itemize}
    \item \textbf{Identification of the Financial Entropy Trap and the Representation Gap.} We theoretically characterize a critical failure mode of maximum-entropy RL in low-SNR environments, demonstrating how stochastic exploration systematically amplifies market noise. Furthermore, we provide a principled analysis of why existing filtering and regularization methods fail to balance noise suppression with information preservation.
    \item \textbf{Proposal and Theoretical Grounding of the FPQC-SAC Architecture.} Moving beyond heuristic module combinations, we design FPQC-SAC by embedding a Parameterized Quantum Circuit at the SAC front-end. We theoretically motivate how quantum entanglement and interference inherently provide a structural constraint that dynamically cancels unstructured noise while preserving temporally correlated weak signals, offering a rigorous decoupling mechanism that traditional architectures lack.
    \item \textbf{Empirical validation and mechanism explanation.} Rigorous backtesting on three real-world datasets shows that FPQC-SAC improves performance over standard SAC by at least 50\% and outperforms the best baseline model by more than 20\%. Ablation studies further verify that this performance leap originates from the trainable quantum entanglement structure and trainable quantum rotation matrices.
\end{itemize}

% Methodology and Theoretical Foundations
% ------------------------------------------------------------------------

\section{Related Work}

\subsection{Financial RL under Low-SNR Market Dynamics}
In quantitative finance, continuous portfolio management and asset allocation are canonically formulated as Markov Decision Processes (MDPs). FinRL~\cite{liu2020finrl} provides an open-source DRL library for automated stock trading, and FinRL-Meta~\cite{liu2022finrlmeta} further expands this line into market environments and benchmarks for data-driven financial RL. Early portfolio-management work by Jiang et al.~\cite{jiang2017deep} formulates deep RL for financial portfolio allocation, while Zhang et al.~\cite{zhang2019deep} studies DRL-based trading in broader market settings. These applications commonly rely on standard continuous-control algorithms, including DDPG~\cite{lillicrap2020continuous}, TD3~\cite{fujimoto2018addressing}, and SAC~\cite{haarnoja2018soft}. As summarized in the survey by Pippas et al.~\cite{pippas2024evolution}, on-policy methods are often favored for stability in volatile markets, whereas off-policy methods offer stronger sample efficiency and exploration. Restoring the training stability of off-policy continuous control methods in financial environments therefore remains a critical challenge. However, existing financial RL pipelines typically treat the representation module as a standard neural encoder, leaving open how low-SNR market noise should be structurally constrained before value estimation.

\subsection{Value-Level Stabilization in Noisy Off-Policy Learning}
A primary vulnerability of off-policy Q-learning is overestimation bias, which arises when noisy Q-value estimates are amplified during Bellman bootstrapping. Fujimoto et al.~\cite{fujimoto2018addressing} analyze this function-approximation error and introduce TD3, which uses clipped double-Q learning to reduce overestimation in continuous control. SAC~\cite{haarnoja2018soft} instead stabilizes off-policy learning through a maximum-entropy objective that encourages stochastic exploration. Later value-level methods further modify the critic side: REDQ~\cite{chen2021randomized} uses randomized critic ensembles to improve sample efficiency, whereas TQC~\cite{kuznetsov2020controlling} truncates distributional quantile critics to control excessive value estimates.

However, when applied to financial time series characterized by inherently low signal-to-noise ratios (SNR), value-level corrections can be insufficient because noisy states may already have contaminated the learned representation before critic-side mitigation is applied. Lyle et al.~\cite{lyle2024normalization} show that normalization and effective learning rates can substantially affect RL update dynamics, and Tran et al.~\cite{tran2026hybrid} revisit overestimation bias and stable policy optimization in hybrid-action actor-critic learning. These studies highlight that critic stabilization is sensitive to the upstream representation and optimization pipeline. This motivates our focus on structural interventions at the geometric representation bottleneck level: value-level corrections mitigate the symptoms of noisy bootstrapping, but they do not prevent noisy financial states from entering the value function in the first place.

\subsection{Quantum Representations for Robust Feature Learning}
Data re-uploading enables PQCs to repeatedly encode classical inputs and act as expressive models for continuous variables, as shown by Pérez-Salinas et al.~\cite{perez2020data}. This expressivity has encouraged the use of Quantum Reinforcement Learning (QRL) in continuous-control domains. Acuto et al.~\cite{acuto2022variational} instantiate a variational quantum SAC architecture for robotic arm control. Hohenfeld et al.~\cite{hohenfeld2024quantum} and Lokossou et al.~\cite{lokossou2025quantum} evaluate quantum deep RL variants for robot-navigation tasks. On the algorithmic side, Jerbi et al.~\cite{jerbi2023quantum} develop quantum policy-gradient algorithms, while Xu and Aggarwal~\cite{xu2025accelerating} accelerate QRL with a quantum natural policy-gradient approach.

Nevertheless, most QRL explorations emphasize parameter efficiency, policy optimization, or sample complexity in physical-control simulators. In financial trading, Liu et al.~\cite{liu2025quantum_fintech} combine quantum-enhanced RL with LSTM forecasting signals, but the quantum component is still primarily used as a nonlinear feature map for decision making. In this work, we pivot the application of QRL: rather than primarily optimizing for parameter reduction, we use the PQC as a representation-level inductive bias for attenuating stochastic financial noise under low-SNR It\^o dynamics.

\subsection{PQC Expressivity, Trainability, and Geometric Bottlenecks}
The design of quantum neural networks is governed by a trade-off between expressivity and trainability. Caro et al.~\cite{caro2021encoding} derive encoding-dependent generalization bounds for PQCs, emphasizing that the data-encoding choice shapes generalization behavior. Martyniuk et al.~\cite{martyniuk2024quantum} survey quantum architecture search and discuss how circuit structure affects expressivity and trainability. Röseler et al.~\cite{roseler2026how} further study how to identify parameterized circuits that remain both expressible and trainable. In noisy financial environments, excessive expressivity can become a liability because the model may fit stochastic variance rather than weak predictive structure. Our architecture therefore adopts a shallow PQC and treats its restricted capacity as an intentional geometric bottleneck, regularizing the representation before actor-critic value estimation.

\section{Methodology and Theoretical Foundations}
\label{sec:methodology}

This section systematically establishes why the quantum representation layer provides a robust structural mechanism to mitigate Q-value overestimation within the Soft Actor-Critic (SAC) framework. We first formalize the extreme low Signal-to-Noise Ratio (SNR) dynamics of financial markets as stochastic differential equations (SDEs) governed by high-dimensional Brownian motion, clarifying the core representational difficulty. Crucially, we shift the paradigm from classical value-clipping to geometric interference. We introduce the network topology and progressively analyze the core structural effects of the Parameterized Quantum Circuit (PQC) bottleneck---hyperspherical phase encoding, entanglement-driven isotropic noise cancellation, and constructive signal alignment---demonstrating how this architecture intrinsically dissipates Temporal Difference (TD) target variance rather than merely enforcing hard boundaries.

\subsection{Low-SNR It\^o SDE and the Representation Challenge}
\label{sec:ito_sde}

In the continuous-time limit, the state variables or asset prices in financial markets can be modeled as multi-dimensional It\^o diffusion processes. 

\begin{assumption}[Low-SNR Local It\^o Transition]\label{ass:ito}
Let the state vector $X_t \in \mathbb{R}^d$ evolve according to the following stochastic differential equation (SDE):
\begin{equation}
    dX_t = \mu(X_t, A_t) dt + \sigma(X_t) dW_t
\end{equation}
where $\mu(\cdot)$ is the drift function carrying the directional predictive signal, $\sigma(\cdot)$ is the instantaneous volatility matrix, and $W_t$ is a $d$-dimensional standard Brownian motion. In the financial low-SNR regime, the critical characteristic is that the magnitude of the drift is vastly eclipsed by the diffusion term:
\begin{equation}
    \|\mu(X_t, A_t) dt\| \ll \|\sigma(X_t) dW_t\|
\end{equation}
\end{assumption}
Although the Brownian increment $dW_t$ possesses independent increments and a zero mean, its sample paths exhibit infinite variation over any arbitrarily small time interval. This implies that on the discrete timescales of algorithmic observation and policy updates, the spatial perturbation of state $X_t$ is mathematically unbounded in Euclidean space.

The direct consequence of this formulation is catastrophic for standard deep reinforcement learning. If unconstrained raw states are fed directly into classical Critic networks, the high-variance diffusion term is injected unmitigated into value function estimation. This causes the bootstrapped Bellman Temporal Difference (TD) targets to fluctuate violently. This mechanism is the representational root of the \textit{financial entropy trap}: noise inherently pollutes the state representation, distorts the Q-value gradients, and is progressively amplified by self-bootstrapping. We provide a mathematical formalization of the \textbf{Financial Entropy Trap} in \textbf{Appendix~\ref{app:entropy_trap}}.

\subsection{FPQC-SAC Architecture and End-to-End Workflow}
\label{sec:architecture}

To fundamentally resolve the Financial Entropy Trap, we must separate the weak drift $\mu$ from the dominating diffusion $\sigma dW_t$. Simply clipping the input or bounding the output is insufficient, as it rigidly squashes both noise and signal. Instead, we embed a Parameterized Quantum Circuit (PQC) as a shared representation bottleneck to naturally filter noise via complex hyperspherical interference. The end-to-end data flow is detailed in Figure~\ref{fig:FPQC-SAC_architecture}.

\begin{figure}[htbp]
  \centering
  \includegraphics[
    width=\linewidth,
     trim=0cm 0cm 0cm 0cm,
    clip
  ]{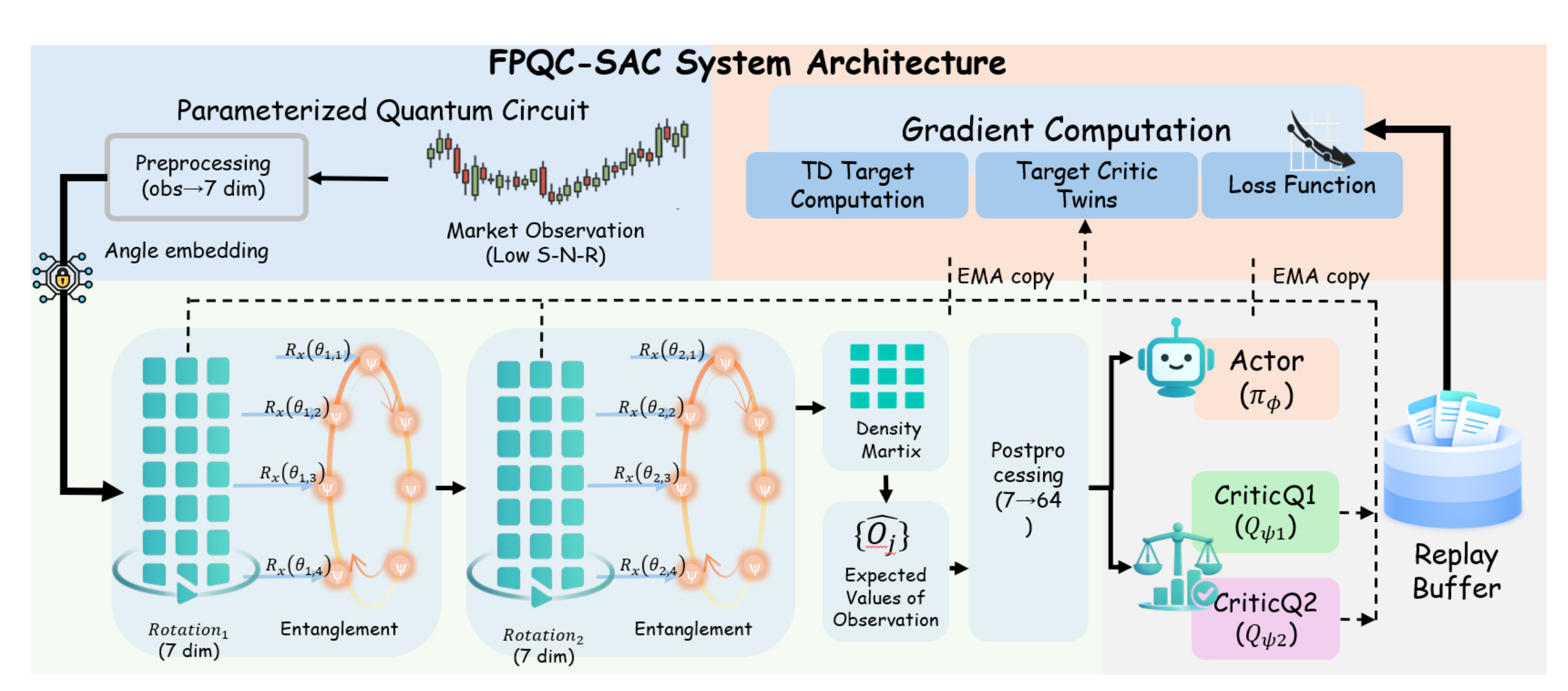}
    \caption{The system architecture and end-to-end workflow of FPQC-SAC. By forcing high-variance irregular Brownian paths through a 7-qubit, 2-layer circularly entangled unitary space, the architecture maps Euclidean noise onto a complex hypersphere. Here, isotropic noise undergoes destructive interference, naturally filtering extreme perturbations before they reach the classical Actor-Critic heads.}
  \label{fig:FPQC-SAC_architecture}
\end{figure}

\textbf{Hybrid Quantum-Classical Architecture and Hyperspherical Isolation}

The framework maps high-noise financial observations $(A_t, X_t)$ onto a stabilized latent manifold. The state space is initially compressed via a classical linear layer to a 7-dimensional vector, aligning with the qubit topology. These features are subsequently embedded into the quantum Hilbert space via continuous $R_x$ angle encoding. The unitary evolution $U(X, \theta)$ utilizes $L=2$ layers of circular CNOT entanglement and single-qubit rotations. The resulting latent representation $f_t$ is derived from the expectation values of Pauli-Z observables, as formalized in Definition~\ref{def:pqc_mapping}.

\begin{definition}[PQC Latent Mapping]\label{def:pqc_mapping}
Let $U(X, \theta) = \prod_{l=1}^L W_l(\theta_l) R_x(X)$ denote the parameterized quantum circuit. The latent representation $f(X, \theta) \in \mathbb{R}^k$ is extracted via the expected values of Pauli observables $\{\hat{O}_j\}_{j=1}^k$ following Born's rule:
\begin{equation}
    f_j(X, \theta) = \text{Tr}\left( \hat{O}_j U(X, \theta) |0\rangle\langle0| U^\dagger(X, \theta) \right).
\end{equation}
\end{definition}

In the reinforcement learning phase, the latent features $f_t$ serve as inputs for the stochastic Actor policy $\pi_\phi$ and the Twin Critics $Q_{\psi_{1,2}}$. Crucially, an Exponential Moving Average (EMA) copy of the PQC parameters ($\bar{\theta}$) is included within the target network cascade. This ensures that the Bellman bootstrapping process evaluates states strictly through the lens of the quantum interference bottleneck, preventing raw environmental variance from infiltrating the non-smooth $\min(Q_{\bar{\psi}_1}, Q_{\bar{\psi}_2})$ operator. Detailed explanations of this interference-driven noise suppression are provided in Sections~\ref{sec:phase_attenuation}, \ref{sec:entanglement_role}, and \ref{sec:geometric_cancellation}.

\subsection{Expectation Attenuation via Complex Phase Encoding}
\label{sec:phase_attenuation}

Before evaluating the quantum encoding, we must account for the classical Pre-Net. As formally proven in \textbf{Appendix~\ref{lemma:prenet_jacobian}}, passing the raw It\^o diffusion through a piecewise-linear, locally Lipschitz Pre-Net (e.g., ReLU) functionally preserves the local It\^o structure almost everywhere, mapping the effective noise covariance to $\Sigma_Z = J_h \Sigma_X J_h^\top$. The unbounded nature of the Brownian noise propagates into the latent space, setting the stage for PQC attenuation.

The first layer of the PQC's denoising mechanism originates from its fundamental operation of angle encoding (the $R_x$ gates shown in Stage 3), which maps classical real-valued inputs into complex exponential phases. Let the pre-processed latent state evolve according to the It\^o process $dZ_s = \mu_s ds + \sigma_s dW_s$ over the time interval $s \in [t, t+\Delta t]$.

\begin{proposition}[Diffusion Attenuation in Periodic Complex Features]
\label{prop:ito_attenuation}
Given an initial latent state $Z_t$, and assuming the drift $\mu_s$ and diffusion covariance $\Sigma_s = \sigma_s \sigma_s^\top$ are locally stable over $[t, t+\Delta t]$, encoding the It\^o diffusion into a continuous unitary generator mode $\omega$ yields the following expected evolution of the complex phase:
\begin{equation}
\label{eq:ito_phase}
\mathbb{E}[e^{i\omega^\top Z_{t+\Delta t}}] = e^{i\omega^\top Z_t} \cdot \underbrace{\exp\left(i \int_t^{t+\Delta t} \omega^\top \mu_s ds \right)}_{\text{Signal Phase Translation}} \cdot \underbrace{\exp\left(-\frac{1}{2} \int_t^{t+\Delta t} \omega^\top \Sigma_s \omega ds \right)}_{\text{Diffusion Attenuation}}.
\end{equation}
\end{proposition}

This analytical equation (detailed proof provided in \textbf{Appendix~\ref{app:proof_prop1}}) mathematically decouples two distinctly divergent effects:
\begin{itemize}
    \item \textbf{Signal Preservation:} The predictive drift term $\mu_s$ is perfectly preserved as a directional phase translation. The directional information of the causal signal survives cleanly within the argument of the complex exponential.
    \item \textbf{Diffusion Attenuation:} Due to its independent increments and zero mean, the stochastic diffusion $\sigma_s dW_s$ algebraically induces an exponential decay factor weighted by the covariance $\Sigma_s$ under expectation. Intuitively, the isotropic random jumps of Brownian motion are ``rotationally averaged'' across the complex plane---random phase oscillations in all directions mutually cancel out in expectation.
\end{itemize}

The core mathematical insight here is that the complex exponential mapping naturally separates the deterministic trend of the signal from the random diffusion of the noise in an expected sense. Signals are retained as phase shifts, while noise is mathematically suppressed as an attenuation factor.
\subsection{The Core Role of Entanglement: Synergistic Signal-to-Noise Separation}
\label{sec:entanglement_role}

The true filtering capability of the PQC does not stem merely from its finite output range, but from how entanglement structurally discriminates between noise and signal within the hypersphere. This functionality manifests in two complementary dimensions:

\textbf{Scattering Noise into Isotropic Dimensions (Destructive Interference).} 
If the PQC solely relied on complex phase encoding without inter-qubit interactions, noise would remain isolated within single-qubit subspaces. Introducing the circular CNOT chain instantaneously diffuses local, independent Brownian perturbations into a collective, multi-body entangled state. Because the financial diffusion term $\sigma dW_t$ is stochastic and largely unstructured, entanglement scrambles these errors into a high-dimensional, isotropic phase perturbation. Due to the projective nature of Pauli measurements (the trace operation), these uniformly scattered random phases destructively interfere and statistically cancel each other out upon readout. Classical sine/cosine bases cannot replicate this synergistic, high-dimensional cancellation.

\textbf{Aligning Signals along Principal Directions (Constructive Interference).} 
While noise is scattered isotropically, the weak drift signal $\mu dt$ possesses intrinsic temporal and cross-asset correlations (e.g., subtle momentum continuations). Under the tuning of trainable rotation gates, the circuit does not scatter these correlated phases uniformly. Instead, it dynamically rotates the hyperspherical basis to align these coherent phase evolutions along specific Pauli measurement directions. Ablation results in Table~\ref{tab:core_ablation} demonstrate that removing CNOT gates or freezing rotations destroys performance. Simply put, entanglement treats noise and signal asymmetrically: it scrambles isotropic noise into destructive cancellation, while combing the structural signal into a focused, constructively interfering output. For a detailed mathematical analysis of these entangled modes, see Appendix~\ref{app:entanglement_modes}.

\subsection{Geometric Consequence: Variance Attenuation via Hyperspherical Cancellation}
\label{sec:geometric_cancellation}

To establish a mathematically rigorous safeguard for off-policy learning, we formalize the consequence of the aforementioned phase interference. While classical bottlenecks (like Tanh) merely clip extreme values---leaving the variance dense at the boundaries---the PQC fundamentally attenuates the variance of the underlying stochastic process.

\begin{proposition}[Hyperspherical Phase Cancellation and Variance Attenuation]
\label{prop:interference_attenuation}
Under unitary evolution, the state $\rho$ is confined to a norm-preserving complex hypersphere. When subjected to the unbounded local It\^o diffusion $dW_t$, the spatial variance is mapped to phase fluctuations. During the expected measurement over Pauli observables, isotropic phase variations lacking structural correlation undergo destructive interference. Consequently, for an input spatial perturbation with variance $\Sigma_X$, the variance of the resulting PQC representation, $\text{Var}(f_{\text{PQC}})$, is exponentially attenuated relative to the intensity of the isotropic noise, completely decoupling the output stability from the infinite-variation nature of the raw SDE.
\end{proposition}

Because this attenuation derives geometrically from the trace operation over an entangled complex manifold, it provides a superior structural guarantee compared to classical bounds. 

\begin{corollary}[Stabilized Representation Cascade]
\label{cor:hybrid_attenuation}
In the hybrid FPQC-SAC architecture, even if the classical Pre-Net amplifies the raw market noise, the subsequent projection through the PQC bottleneck forces the isotropic components to destructively cancel. Assuming a locally Lipschitz Post-Net, the final representation perturbation $\|f(X) - f(\bar{X})\|_2$ exhibits a tightly concentrated variance, immune to the catastrophic Brownian spikes that bypass standard Euclidean neural encoders.
\end{corollary}

In this context, the quantum bottleneck acts not as a simple hard-clipping wall, but as an active representational dissipater that absorbs and neutralizes high-variance unstructured noise.

\subsection{Structural Regularization of TD Targets via Variance Dissipation}
\label{sec:td_bound}

We now connect this geometric cancellation to the most noise-sensitive component of SAC: the Temporal Difference (TD) target. The vulnerability of standard maximum-entropy RL lies in the operator $y_t = r_t + \gamma \max_{a'} Q(X_{t+1}, a')$. When $X_{t+1}$ contains unmitigated Brownian noise, the maximization over noisy representations systematically induces a positive bias, inflating the Q-values.

In our framework, the bootstrapped target is computed as:
\begin{equation}
\label{eq:sac_td_target}
y_t = r_t + \gamma \mathcal{B} \left( f(X_{t+1}) \right),
\end{equation}
where $\mathcal{B}$ denotes the locally Lipschitz soft Bellman target map. 

Because $f(X_{t+1})$ has passed through the PQC, its isotropic noise components have been dissipated via destructive interference (Proposition~\ref{prop:interference_attenuation}). Consequently, the variance of the input to the Bellman operator is drastically suppressed: $\text{Var}(f(X_{t+1})) \ll \text{Var}(X_{t+1})$. 

This means the amplitude of the spurious spikes entering $\mathcal{B}(\cdot)$ is structurally minimized. Without extreme positive noise spikes, the $\max$ (or soft-max) operator has no artificial high values to exploit. Thus, FPQC-SAC systematically starves the overestimation bias of its necessary fuel (high-variance perturbations). By dissipating the noise \textit{before} Bellman bootstrapping rather than trying to regularize the Q-values \textit{after} the fact, the architecture successfully bypasses the amplification channel of the Financial Entropy Trap.

\subsection{Closing the Loop: Theoretical Mechanisms Meet Empirical Validation}
\label{sec:theory_meets_exp}

The mechanisms elucidated in Sections~\ref{sec:phase_attenuation}--\ref{sec:td_bound}---complex hyperspherical encoding, entanglement-driven destructive/constructive interference, and variance dissipation---are structurally unified. Phase encoding converts spatial noise to phase variation, entanglement scrambles it isotropically, and quantum measurement neutralizes it while preserving the coherently aligned drift signals. 

To rigorously guarantee that this interference-based filtering does not squash causal signals alongside the noise, \textbf{Appendix~\ref{app:snr_bound}} derives a theoretical lower bound for the Output Signal-to-Noise Ratio ($\text{SNR}_{\text{out}}$). We mathematically demonstrate that while unstructured noise variance is aggressively dissipated, the sensitivity to deterministic structural signals remains strictly bounded away from zero. 

This mathematically justified mechanism formulates a verifiable hypothesis: the PQC acts as an interference-driven filter that separates signal from noise, thus stabilizing value estimation. Therefore, Section~\ref{sec:experiments} evaluates FPQC-SAC through a complete empirical causal chain:
\[
\begin{array}{c}
\text{Interference-driven signal-preserving bottleneck}\\
\text{(Sections~\ref{sec:phase_attenuation}--\ref{sec:td_bound})}\\[0.35em]
\downarrow\\[0.35em]
\text{Suppressed TD variance and minimal Q-gap}\\
\text{(Section~\ref{sec:q_value_bottleneck})}\\[0.35em]
\downarrow\\[0.35em]
\text{Superior out-of-sample trading performance}\\
\text{(Section~\ref{sec:trading_performance})}.
\end{array}
\]
Specifically, we first test whether the PQC bottleneck can suppress TD-target variance and reduce the Q-gap without destroying structural information (outperforming simple classical bounded activation like Tanh). Subsequently, we evaluate whether this intrinsic Bellman stability directly translates into superior out-of-sample trading returns. Through this closed loop of theoretical mechanisms, algorithmic stability, and final financial performance, the superiority of the hyperspherical cancellation paradigm is thoroughly validated.

\section{Experiments}
\label{sec:experiments}
\subsection{Experimental Setup}
\label{sec:experimental_setup}
We evaluate FPQC-SAC on three non-overlapping U.S. multi-asset portfolios, covering 18 distinct assets across benchmark, defensive, and high-volatility regimes. The first portfolio (AAPL, AMZN, GOOGL, MSFT, QQQ, SPY) follows the canonical FinRL portfolio-allocation setting~\cite{liu2020finrl}, providing a direct comparison with the standard financial RL benchmark rather than a post-hoc asset selection. To test whether the observed gains extend beyond this benchmark regime, we further construct two supplementary six-asset portfolios: a stable blue-chip portfolio (BAC, JNJ, JPM, PG, WMT, XOM) with persistent signals and lower volatility, and a high-growth technology portfolio (AVGO, BRK.B, META, NFLX, NVDA, TSLA) with concentrated growth exposure and elevated regime sensitivity. This design preserves the six-asset evaluation scale commonly used in FinRL-style portfolio allocation while extending the evaluation across distinct market conditions. All DRL agents are trained from 2013-01-02 to 2018-12-31 for 25,000 environment steps and tested out-of-sample from 2019-01-02 to 2021-08-31, covering the COVID-19 shock and recovery. Metrics are reported as mean$\pm$std over independent seeds using a 10\% two-sided trimmed mean, removing the two highest and two lowest cumulative-return seeds from 20 runs to improve robustness against outlier runs~\cite{mair2020robust}. Detailed parameters are provided in Appendix~\ref{appendix:repro_details}.
Baselines are divided into three categories: (1) classic continuous-control DRL algorithms (SAC, TD3, PPO, DDPG, A2C); (2) output-regularized robust DRL methods for mitigating Q-value overestimation (REDQ, TQC); (3) rule-based trading strategies, including MACD, SMA, KDJ-RSI, and ZMR.

\subsection{Q-Value Overestimation Mitigation and Bottleneck Analysis}
\label{sec:q_value_bottleneck}

In this section, we respond to the theoretical hypothesis proposed in Section~\ref{sec:theory_meets_exp} by testing whether the PQC representation constraint can suppress TD-target variance and reduce the Q-gap. During the out-of-sample testing phase, we measured the deviation between the learned Critic values and the Monte Carlo soft returns (Q-gap), alongside Temporal Difference variance (TD-var) and Critic-loss variance (C-loss-var). Furthermore, to definitively rule out the alternative hypothesis that "any structural representation bottleneck could achieve similar stabilizing effects," we systematically substituted the PQC with a suite of classical representation designs. These include bounded activations (Tanh), feature clipping (Clipped Latent), dimensionality compression (Linear and MLP Bottlenecks), generic regularization (Weight Decay), and normalization techniques (LayerNorm and SpectralNorm). Each bottleneck is embedded within the identical Actor-Critic backbone and trained under the exact same protocol.
\begin{table}[htbp]
  \centering
  \caption{Systematic comparison of classical representation bottlenecks and FPQC-SAC on return-stability metrics.}
  \label{tab:bottleneck_comparison}
  \resizebox{\textwidth}{!}{
  \begin{tabular}{lcccccccc}
    \toprule
    \textbf{Model} & \textbf{CR\%$\uparrow$} & \textbf{AR\%$\uparrow$} & \textbf{SR$\uparrow$} & \textbf{Sortino$\uparrow$} & \textbf{Calmar$\uparrow$} & \textbf{Q-gap$\downarrow$} & \textbf{TD-var$\downarrow$} & \textbf{C-loss-var$\downarrow$} \\
    \midrule
    Linear Bottleneck \citep{pearson1901lines} & 98.9$\pm$60.1 & 27.8$\pm$16.8 & 1.18$\pm$0.63 & 1.31$\pm$0.75 & 1.06$\pm$0.56 & -0.553$\pm$2.16 & $4.16{\times}10^3\pm7.49{\times}10^3$ & $4.71{\times}10^5\pm1.20{\times}10^6$ \\
    SAC (No Bottleneck) \citep{haarnoja2018soft} & 117.0$\pm$18.6 & 33.7$\pm$4.29 & 1.40$\pm$0.12 & 1.59$\pm$0.19 & 1.11$\pm$0.23 & 26.4$\pm$81.8 & $2.55{\times}10^7\pm1.01{\times}10^8$ & $1.32{\times}10^{13}\pm4.78{\times}10^{13}$ \\
    Weight Decay Bot. \citep{krogh1991simple} & 125.0$\pm$33.4 & 35.4$\pm$7.54 & 1.43$\pm$0.19 & 1.71$\pm$0.26 & 1.31$\pm$0.32 & -0.148$\pm$0.30 & 23.6$\pm$31.6 & 157$\pm$277 \\
    MLP Bottleneck \citep{hornik1989multilayer,hinton2006reducing} & 136.0$\pm$29.2 & 37.9$\pm$6.38 & 1.53$\pm$0.18 & 1.77$\pm$0.30 & 1.31$\pm$0.33 & -0.047$\pm$0.14 & 25.1$\pm$35.1 & 140$\pm$334 \\
    SpectralNorm Bot. \citep{miyato2018spectral} & 144.0$\pm$65.2 & 38.4$\pm$16.4 & 1.44$\pm$0.57 & 1.73$\pm$0.71 & 1.38$\pm$0.60 & $1.92{\times}10^{12}\pm8.73{\times}10^{12}$ & $2.91{\times}10^{27}\pm1.30{\times}10^{28}$ & $2.99{\times}10^{57}\pm1.34{\times}10^{58}$ \\
    Tanh Bottleneck \citep{lecun2002efficient} & 154.0$\pm$30.4 & 41.6$\pm$6.35 & 1.57$\pm$0.16 & 1.88$\pm$0.20 & 1.46$\pm$0.25 & -0.004$\pm$0.05 & 0.81$\pm$1.02 & 14.8$\pm$28.4 \\
    Clipped Latent \citep{choi2018pact,xu2017feature} & 155.0$\pm$30.1 & 41.9$\pm$6.20 & 1.55$\pm$0.21 & 1.91$\pm$0.19 & 1.54$\pm$0.20 & -0.007$\pm$0.05 & 0.93$\pm$1.65 & 18.7$\pm$46.6 \\
    LayerNorm Bot. \citep{ba2016layer} & 170.0$\pm$36.8 & 44.9$\pm$7.52 & 1.63$\pm$0.21 & 2.00$\pm$0.26 & 1.64$\pm$0.29 & -0.002$\pm$0.05 & 0.76$\pm$0.74 & 8.99$\pm$14.3 \\
    \midrule
    \textbf{FPQC-SAC (Ours)} & \textbf{195.3$\pm$26.9} & \textbf{50.0$\pm$5.15} & \textbf{1.71$\pm$0.15} & \textbf{2.13$\pm$0.19} & \textbf{1.75$\pm$0.20} & \textbf{0.0004$\pm$0.05} & \textbf{0.54$\pm$0.69} & \textbf{8.19$\pm$13.7} \\
    \bottomrule
  \end{tabular}
  }
\end{table}
Table~\ref{tab:bottleneck_comparison} reveals key findings. Unconstrained SAC suffers catastrophic Q-value explosions and extreme estimation variance, which distort policy gradients and mislead the actor. To address this, our analysis focuses fundamentally on the role of bounded features. Bounded constraints themselves are not rare. Classical operations, e.g., tanh activations, gradient clipping, and LayerNorm can also impose range restrictions on feature spaces. While these classical bottlenecks exhibit a clear return-stability trade-off---with LayerNorm and Tanh variants achieving moderate returns and partial stability gains---they cannot fully suppress noise while retaining weak predictive signals. Furthermore, we note that some ill-designed representation constraints exhibit extremely high estimation variance, which is primarily attributed to catastrophic training seed explosions.

Only PQC simultaneously delivers high returns and stability. Crucially, addressing our theoretical propositions from Section~\ref{sec:methodology}, FPQC-SAC achieves the lowest TD-variance (0.54), Q-gap (0.0004), and Critic-loss-variance (8.19) across all tested architectures. It significantly outperforms the most competitive classical bounded baselines, further reducing the TD-variance compared to LayerNorm (0.76) and Tanh (0.81), while achieving a much tighter Q-gap approximation than Clipped Latent (0.007). These stability advantages seamlessly translate into performance, allowing FPQC-SAC to reach a 195.3\% cumulative return, comfortably outperforming the runner-up LayerNorm Bottleneck (170.0\%). Ultimately, PQC's unique combination of unitary-constrained bounded readout and entanglement-induced phase cancellation creates a stable architecture not easily replicated by conventional designs. For more detailed variance, explained-variance, Q-gap, and TD-var analyses of representative spectral/filtering encoders, please see Appendix~\ref{appendix:encoder_diagnostics}.

\subsection{Trading Performance Comparison}
\label{sec:trading_performance}
Having shown that the PQC bottleneck stabilizes value estimation in Section~\ref{sec:q_value_bottleneck}, we now test whether this stability translates into out-of-sample trading performance. Our evaluation uses one benchmark experiment and two regime-extension experiments.

For the \textbf{Main Experiment}, we use the canonical FinRL~\cite{liu2020finrl} 6-asset portfolio (AAPL, AMZN, GOOGL, MSFT, QQQ, SPY), which combines large-cap technology stocks with broad market ETFs. We extend the original FinRL test horizon to cover the COVID-19 shock and recovery, and compare FPQC-SAC against both standard continuous-control baselines and the SAC family, including SAC, REDQ, TQC, DroQ, and CrossQ.

To test generalization beyond the FinRL benchmark, we add two non-overlapping six-asset portfolios:

\textbf{Sub-experiment 1 (Defensive Blue-chip)} uses BAC, JNJ, JPM, PG, WMT, and XOM to represent a lower-volatility regime with more persistent signals and stronger downside-risk considerations.

\textbf{Sub-experiment 2 (High-Volatility Growth)} uses NVDA, META, AVGO, TSLA, NFLX, and BRK.B to stress-test models under elevated volatility, growth concentration, and severe stochastic interference. Together, the three portfolios cover 18 distinct assets across benchmark, defensive, and high-volatility regimes. For all experiments, metrics follow the 10\% trimmed-mean protocol over 20 random seeds.
\begin{table}[htbp]
  \centering
  \caption{Main experiment on the Mainstream Tech \& Market-Index Portfolio (QQQ, SPY, MSFT, GOOGL, AMZN, AAPL). Red marks the best result and blue marks the second-best.}
  \label{tab:main_experiment}
  \small
  \setlength{\tabcolsep}{4pt}
  \resizebox{\textwidth}{!}{
  \begin{tabular}{llccccc}
    \toprule
    \textbf{Category} & \textbf{Model} 
    & \textbf{CR\%$\uparrow$} & \textbf{AR\%$\uparrow$} & \textbf{SR$\uparrow$} & \textbf{Sortino$\uparrow$} & \textbf{Calmar$\uparrow$} \\
    \midrule

    Market & B\&H 
    & 88.57$\pm$0.00 & 26.90$\pm$0.00 & 1.18$\pm$0.00 & 1.28$\pm$0.00 & 0.80$\pm$0.00 \\
    \midrule

    Our Model & \textbf{FPQC-SAC}
    & \color{red}\textbf{195.32$\pm$26.94} & \color{red}\textbf{50.05$\pm$5.15} & \color{red}\textbf{1.71$\pm$0.15} & \color{red}\textbf{2.13$\pm$0.19} & \color{red}\textbf{1.75$\pm$0.20} \\
    \midrule

    \multirow{9}{*}{DRL-based}
    & SAC~\cite{haarnoja2018soft}
    & 117.04$\pm$18.56 & 33.67$\pm$4.29 & 1.40$\pm$0.12 & 1.59$\pm$0.19 & 1.11$\pm$0.23 \\

    & CrossQ~\cite{CrossQ2024}
    & 148.19$\pm$38.63 & 40.32$\pm$8.21 & 1.54$\pm$0.26 & 1.87$\pm$0.34 & 1.53$\pm$0.36 \\

    & DroQ~\cite{DroQ2022}
    & 152.80$\pm$21.61 & 41.55$\pm$4.51 & 1.56$\pm$0.13 & 1.91$\pm$0.15 & 1.56$\pm$0.18 \\

    & REDQ~\cite{chen2021randomized}
    & 147.87$\pm$23.48 & 40.48$\pm$5.16 & 1.57$\pm$0.14 & 1.85$\pm$0.20 & 1.44$\pm$0.24 \\

    & TQC~\cite{kuznetsov2020controlling}
    & 145.10$\pm$23.24 & 39.89$\pm$5.01 & 1.57$\pm$0.11 & 1.82$\pm$0.18 & 1.38$\pm$0.24 \\

    & TD3~\cite{fujimoto2018addressing}
    & 151.55$\pm$29.20 & 41.20$\pm$6.12 & 1.57$\pm$0.12 & 1.87$\pm$0.21 & 1.46$\pm$0.29 \\

    & DDPG~\cite{lillicrap2020continuous}
    & \color{blue}153.65$\pm$27.25 & \color{blue}41.67$\pm$5.68 & \color{blue}1.60$\pm$0.12 & \color{blue}1.90$\pm$0.20 & \color{blue}1.48$\pm$0.28 \\

    & PPO~\cite{schulman2017proximal}
    & 151.94$\pm$27.50 & 41.31$\pm$5.72 & 1.56$\pm$0.15 & 1.89$\pm$0.20 & 1.48$\pm$0.21 \\

    & A2C~\cite{mnih2016asynchronous}
    & 145.76$\pm$21.32 & 40.06$\pm$4.57 & 1.51$\pm$0.14 & 1.84$\pm$0.16 & 1.46$\pm$0.20 \\
    \midrule

    \multirow{4}{*}{Rule-based}
    & MACD & 14.55$\pm$0.00 & 5.23$\pm$0.00 & 0.39$\pm$0.00 & 0.38$\pm$0.00 & 0.28$\pm$0.00 \\
    & SMA & 49.47$\pm$0.00 & 16.29$\pm$0.00 & 0.87$\pm$0.00 & 0.83$\pm$0.00 & 0.70$\pm$0.00 \\
    & KDJ\_RSI & 0.86$\pm$0.00 & 0.32$\pm$0.00 & 0.74$\pm$0.00 & 0.17$\pm$0.00 & 0.58$\pm$0.00 \\
    & ZMR & 43.12$\pm$0.00 & 14.41$\pm$0.00 & 0.82$\pm$0.00 & 0.85$\pm$0.00 & 0.52$\pm$0.00 \\
    \bottomrule
  \end{tabular}
  }
\end{table}

\begin{table}[htbp]
  \centering
  \caption{Generalization experiments across defensive blue-chip and high-volatility growth portfolios. Red marks the best result and blue marks the second-best.}
  \label{tab:generalization_experiments}
  \small
  \setlength{\tabcolsep}{3pt}
  \resizebox{\textwidth}{!}{
  \begin{tabular}{llccccc|ccccc}
    \toprule
    \multirow{2}{*}{\textbf{Category}} 
    & \multirow{2}{*}{\textbf{Model}} 
    & \multicolumn{5}{c|}{\textbf{Defensive Blue-chip Portfolio (Sub-experiment 1)}} 
    & \multicolumn{5}{c}{\textbf{High-Volatility Growth Portfolio (Sub-experiment 2)}} \\
    \cmidrule(lr){3-7} \cmidrule(lr){8-12}
    & & CR\%$\uparrow$ & AR\%$\uparrow$ & SR$\uparrow$ & Sortino$\uparrow$ & Calmar$\uparrow$
      & CR\%$\uparrow$ & AR\%$\uparrow$ & SR$\uparrow$ & Sortino$\uparrow$ & Calmar$\uparrow$ \\
    \midrule

    Our Model & \textbf{FPQC-SAC}
    & \color{blue}45.01$\pm$13.58 & \color{blue}14.86$\pm$4.20 & \color{red}\textbf{0.62$\pm$0.21} & \color{red}\textbf{0.81$\pm$0.28} & \color{red}\textbf{0.64$\pm$0.33}
    & \color{red}\textbf{183.26$\pm$75.58} & \color{red}\textbf{46.77$\pm$14.20} & \color{red}\textbf{1.37$\pm$0.30} & \color{red}\textbf{1.76$\pm$0.39} & \color{red}\textbf{1.35$\pm$0.36} \\
    \midrule

    \multirow{7}{*}{Off-policy DRL}
    & SAC~\cite{haarnoja2018soft}
    & 27.25$\pm$24.62 & 8.97$\pm$8.45 & 0.41$\pm$0.35 & 0.51$\pm$0.44 & 0.33$\pm$0.29 
    & 133.33$\pm$54.18 & 36.72$\pm$11.37 & 1.14$\pm$0.33 & 1.46$\pm$0.39 & 1.10$\pm$0.29 \\

    & CrossQ~\cite{CrossQ2024}
    & \color{red}\textbf{47.15$\pm$6.11} & \color{red}\textbf{15.59$\pm$1.79} & \color{blue}0.62$\pm$0.11 & \color{blue}0.76$\pm$0.13 & \color{blue}0.50$\pm$0.13
    & 87.31$\pm$28.66 & 26.26$\pm$7.13 & 1.02$\pm$0.22 & 1.20$\pm$0.28 & 0.87$\pm$0.26 \\

    & DroQ~\cite{DroQ2022}
    & 31.71$\pm$7.42 & 10.86$\pm$2.33 & 0.50$\pm$0.09 & 0.60$\pm$0.12 & 0.35$\pm$0.07
    & 111.37$\pm$40.43 & 31.93$\pm$9.43 & 1.11$\pm$0.30 & 1.37$\pm$0.34 & 1.01$\pm$0.25 \\

    & REDQ~\cite{chen2021randomized}
    & 41.47$\pm$15.48 & 13.75$\pm$4.96 & 0.58$\pm$0.22 & 0.74$\pm$0.28 & 0.49$\pm$0.22
    & 136.65$\pm$82.74 & 38.20$\pm$16.31 & 1.10$\pm$0.30 & 1.50$\pm$0.37 & 1.10$\pm$0.27 \\

    & TQC~\cite{kuznetsov2020controlling}
    & 30.02$\pm$31.07 & 9.56$\pm$10.74 & 0.44$\pm$0.41 & 0.54$\pm$0.54 & 0.39$\pm$0.38
    & 118.14$\pm$69.60 & 32.58$\pm$15.75 & 1.15$\pm$0.43 & 1.42$\pm$0.54 & 1.02$\pm$0.41 \\

    & TD3~\cite{fujimoto2018addressing}
    & 33.68$\pm$11.91 & 11.42$\pm$3.77 & 0.51$\pm$0.18 & 0.62$\pm$0.21 & 0.35$\pm$0.14
    & \color{blue}154.99$\pm$58.16 & \color{blue}41.36$\pm$11.78 & \color{blue}1.36$\pm$0.28 & \color{blue}1.70$\pm$0.35 & \color{blue}1.25$\pm$0.26 \\

    & DDPG~\cite{lillicrap2020continuous} 
    & 35.85$\pm$11.70 & 12.10$\pm$3.63 & 0.55$\pm$0.19 & 0.67$\pm$0.23 & 0.41$\pm$0.20
    & 144.49$\pm$39.61 & 39.45$\pm$9.00 & 1.30$\pm$0.25 & 1.64$\pm$0.30 & 1.21$\pm$0.21 \\
    \midrule

    \multirow{4}{*}{Rule-based}
    & MACD
    & 4.11$\pm$0.00 & 1.52$\pm$0.00 & 0.25$\pm$0.00 & 0.29$\pm$0.00 & 0.11$\pm$0.00
    & 22.16$\pm$0.00 & 7.81$\pm$0.00 & 0.83$\pm$0.00 & 0.95$\pm$0.00 & 0.71$\pm$0.00 \\
    
    & SMA
    & 5.16$\pm$0.00 & 1.91$\pm$0.00 & 0.30$\pm$0.00 & 0.29$\pm$0.00 & 0.14$\pm$0.00
    & 37.34$\pm$0.00 & 12.65$\pm$0.00 & 0.99$\pm$0.00 & 0.97$\pm$0.00 & 0.80$\pm$0.00 \\
    
    & KDJ\_RSI
    & -1.82$\pm$0.00 & -0.69$\pm$0.00 & -0.12$\pm$0.00 & -0.16$\pm$0.00 & -0.05$\pm$0.00
    & 1.09$\pm$0.00 & 0.41$\pm$0.00 & 0.10$\pm$0.00 & 0.08$\pm$0.00 & 0.06$\pm$0.00 \\
    
    & ZMR
    & 10.90$\pm$0.00 & 3.96$\pm$0.00 & 0.24$\pm$0.00 & 0.28$\pm$0.00 & 0.13$\pm$0.00
    & 66.79$\pm$0.00 & 21.18$\pm$0.00 & 1.10$\pm$0.00 & 1.21$\pm$0.00 & 0.73$\pm$0.00 \\
    \bottomrule
  \end{tabular}
}
\end{table}

Tables \ref{tab:main_experiment} and \ref{tab:generalization_experiments} demonstrate that FPQC-SAC consistently dominates across varied market regimes, explicitly highlighting the structural vulnerabilities of standard SAC improvements. In the \textbf{Main Experiment}, contemporary baselines successfully improve sample efficiency and stability over vanilla SAC, with DroQ~\cite{DroQ2022} (152.80\% CR) strongly competing with prior algorithms. Nevertheless, FPQC-SAC maintains a substantial 42-percentage-point lead (\textbf{195.32\%} CR). This confirms that weak-signal-preserving representation filtering is fundamentally more effective than downstream heuristic regularizations in standard market conditions.

This divergence becomes highly nuanced in the generalization tests. In the defensive \textbf{Sub-experiment 1}, where noise is low and trends are persistent, standard off-policy algorithms typically struggle. Here, CrossQ~\cite{CrossQ2024} adapts exceptionally well, leveraging its cross-Q updates to fit weak trends and achieve the highest absolute return (\textbf{47.15\%} CR). However, this emphasizes a classic risk-reward trade-off: while CrossQ extracts slightly higher absolute profits, FPQC-SAC trades off a marginal amount of return to retain significantly superior downside protection, dominating the risk-adjusted metrics with the highest Sortino (\textbf{0.81}) and Calmar (\textbf{0.64}) ratios. This validates the PQC bottleneck's inherent capability to restrict tail risks even in calm markets.

Most critically, \textbf{Sub-experiment 2} illustrates the \textit{Financial Entropy Trap}. In this highly volatile portfolio, extreme stochastic noise severely penalizes traditional continuous control architectures. Despite their theoretical stability in standard benchmarks, both DroQ and CrossQ suffer severe structural collapse here, dropping to 111.37\% and 87.31\% CR respectively, completely underperforming vanilla SAC (133.33\% CR). This exposes a vulnerability: network-layer perturbations---such as dropout Q-functions~\cite{DroQ2022} or cross-Q updates~\cite{CrossQ2024}---can inadvertently amplify gradient variance when exposed to extreme out-of-distribution market noise. In stark contrast, FPQC-SAC anchors policy updates via its path-wise bounded quantum readout, mitigating the extreme variance and achieving \textbf{183.26\%} CR. These results validate our core hypothesis: while downstream regularizations provide stability and efficiency in standard settings, embedding a hard structural constraint at the representation level is the only mechanism that provides universal robustness across all financial regimes.

\subsection{Ablations and Analysis}
\label{sec:ablations}
This section conducts ablation experiments to identify the core components driving performance gains. The fully trainable FPQC-SAC is benchmarked against two degraded variants: \textbf{Frozen-PQC FPQC-SAC}, which freezes all trainable rotation parameters while preserving the CNOT topology, and \textbf{No-CNOT FPQC-SAC}, which removes all CNOT connections while retaining trainable single-qubit rotations. We further vary PQC depth across \(L \in \{1,2,3,4,5\}\) to explore the influence of circuit expressivity.

\begin{table}[t]
  \centering
  \caption{Ablation studies of FPQC-SAC. Red/blue denote best/second best.}
  \label{tab:core_ablation}
  \small
  \setlength{\tabcolsep}{3pt}
  \resizebox{\textwidth}{!}{
  \begin{tabular}{lccccc}
    \toprule
    \textbf{Model} & CR\%$\uparrow$ & AR\%$\uparrow$ & SR$\uparrow$ & Sortino$\uparrow$ & Calmar$\uparrow$ \\
    \midrule
    \multicolumn{6}{c}{\textbf{Core Structural Ablation}} \\
    \midrule
    \textbf{Trainable}
    & \color{red}\textbf{195.32$\pm$26.94}
    & \color{red}\textbf{50.05$\pm$5.15}
    & \color{red}\textbf{1.71$\pm$0.15}
    & \color{red}\textbf{2.13$\pm$0.19}
    & \color{red}\textbf{1.75$\pm$0.20} \\

    No-CNOT
    & \color{blue}148.49$\pm$69.60
    & \color{blue}40.37$\pm$15.49
    & 1.45$\pm$0.38
    & \color{blue}1.86$\pm$0.52
    & \color{blue}1.54$\pm$0.43 \\

    Frozen-PQC
    & 144.51$\pm$20.21
    & 39.76$\pm$4.31
    & \color{blue}1.51$\pm$0.10
    & 1.81$\pm$0.16
    & 1.39$\pm$0.19 \\
    \midrule
    \multicolumn{6}{c}{\textbf{PQC Depth Ablation}} \\
    \midrule
    L1
    & 150.90$\pm$22.09
    & 41.06$\pm$4.75
    & 1.48$\pm$0.15
    & 1.87$\pm$0.18
    & 1.52$\pm$0.21 \\

    \textbf{L2}
    & \color{red}\textbf{195.32$\pm$26.94}
    & \color{red}\textbf{50.05$\pm$5.15}
    & \color{red}\textbf{1.71$\pm$0.15}
    & \color{red}\textbf{2.13$\pm$0.19}
    & \color{red}\textbf{1.75$\pm$0.20} \\

    L3
    & 152.18$\pm$21.83
    & 41.33$\pm$4.68
    & 1.51$\pm$0.12
    & 1.90$\pm$0.17
    & 1.53$\pm$0.20 \\

    L4
    & \color{blue}164.85$\pm$37.17
    & \color{blue}43.73$\pm$7.52
    & 1.54$\pm$0.19
    & \color{blue}1.94$\pm$0.23
    & \color{blue}1.60$\pm$0.29 \\

    L5
    & 151.48$\pm$27.71
    & 41.15$\pm$5.97
    & \color{blue}1.54$\pm$0.17
    & 1.89$\pm$0.22
    & 1.48$\pm$0.22 \\
    \bottomrule
  \end{tabular}
  }
\end{table}

Table~\ref{tab:core_ablation} shows disabling either parameter trainability or quantum entanglement markedly degrades performance: the cumulative return falls from 195.32\% to 144.51\% for Frozen-PQC and 148.49\% for No-CNOT. The No-CNOT variant also presents substantially higher seed-level variance, indicating entanglement boosts both trading returns and training stability. These outcomes confirm the performance gain stems not from fixed bounded projection or isolated single-qubit rotations alone, but from the joint effect of trainable rotations and entangled circuit topology.

Table~\ref{tab:core_ablation} further shows \(L=2\) yields the best overall performance: \(L=1\) underperforms, while \(L\ge3\) causes consistent degradation. This indicates FPQC-SAC benefits from a shallow yet expressive quantum bottleneck, with sufficient capacity to model useful market structure without overfitting stochastic fluctuations. Combined with the representation analysis in Appendix~\ref{appendix:encoder_diagnostics} and the classical bottleneck comparisons in Table~\ref{tab:bottleneck_comparison}, these ablations confirm FPQC-SAC's improvement stems from its specific shallow, trainable, entangled PQC design---not generic dimensionality reduction or bounded feature compression.

\section{Conclusion}
This paper addresses a theoretical vulnerability when applying maximum-entropy off-policy reinforcement learning to low-SNR financial environments. We identify the ``Financial Entropy Trap'', where maximum-entropy RL can amplify low-SNR noise into systematic Q-value overestimation, and establish that a bounded quantum representation acts as a geometric constraint under local Lipschitz assumptions. Based on this insight, we propose FPQC-SAC, an architectural innovation that relies on a shallow Parameterized Quantum Circuit (PQC) as a structural information bottleneck to map noisy inputs into periodic complex phases for statistical attenuation while leveraging circular entanglement to induce destructive interference of unstructured noise and constructively align weak predictive signals within a bounded unitary space. Extensive empirical evaluations validate this design: FPQC-SAC achieves a 66.89\% relative gain in cumulative return over standard SAC, outperforms the strongest continuous-control DRL baseline by $\sim$27\%, and yields the lowest TD-target variance (0.54). Ultimately, these results underscore a promising paradigm shift: utilizing quantum circuits not merely for computational speedup, but as robust geometric constraints to mitigate representation vulnerabilities in highly stochastic domains.

\bibliographystyle{abbrvnat} 
\bibliography{refer}

@InProceedings{CrossQ2024,
  author       = {Aditya Bhatt and Daniel Palenicek and Boris Belousov and Max Argus and Artemij Amiranashvili and Thomas Brox and Jan Peters},
  title        = {CrossQ: Batch Normalization in Deep Reinforcement Learning for Greater Sample Efficiency and Simplicity},
  booktitle    = {International Conference on Learning Representations (ICLR)},
  year         = {2024},
  url          = {https://proceedings.iclr.cc/paper_files/paper/2024/hash/f381114cf5aba4e45552869863deaaa7-Abstract-Conference.html}
}

@InProceedings{DroQ2022,
  author       = {Takuya Hiraoka and Takahisa Imagawa and Taisei Hashimoto and Takashi Onishi and Yoshimasa Tsuruoka},
  title        = {Dropout Q-Functions for Doubly Efficient Reinforcement Learning},
  booktitle    = {International Conference on Learning Representations (ICLR)},
  year         = {2022},
  url          = {https://openreview.net/forum?id=xCVJMsPv3RT}
}

@article{gu2020empirical,
  title   = {Empirical Asset Pricing via Machine Learning},
  author  = {Gu, Shihao and Kelly, Bryan and Xiu, Dacheng},
  journal = {The Review of Financial Studies},
  volume  = {33},
  number  = {5},
  pages   = {2223--2273},
  year    = {2020},
  doi     = {10.1093/rfs/hhaa009}
}

@article{kong2013cancellation,
  title     = {Cancellation of Internal Quantum Noise of an Amplifier by Quantum Correlation},
  author    = {Kong, Jia and Hudelist, F. and Ou, Z. Y. and Zhang, Weiping},
  journal   = {Physical Review Letters},
  volume    = {111},
  number    = {3},
  pages     = {033608},
  year      = {2013},
  month     = jul,
  publisher = {American Physical Society},
  doi       = {10.1103/PhysRevLett.111.033608}
}

@incollection{lecun2002efficient,
  title={Efficient BackProp},
  author={LeCun, Yann and Bottou, L{\'e}on and Orr, Genevieve B and M{\"u}ller, Klaus-Robert},
  booktitle={Neural Networks: Tricks of the Trade},
  series={Lecture Notes in Computer Science},
  volume={7700},
  pages={9--48},
  year={2012},
  publisher={Springer},
  doi={10.1007/978-3-642-35289-8_3}
}

@article{lillicrap2020continuous,
  title={Continuous Control with Deep Reinforcement Learning},
  author={Lillicrap, Timothy P. and Hunt, Jonathan J. and Pritzel, Alexander and Heess, Nicolas and Erez, Tom and Tassa, Yuval and Silver, David and Wierstra, Daan},
  journal={arXiv preprint arXiv:1509.02971},
  year={2015}
}

@article{bao2017deep,
  title   = {A Deep Learning Framework for Financial Time Series Using Stacked Autoencoders and Long-Short Term Memory},
  author  = {Wei Bao and Jun Yue and Yulei Rao},
  journal = {PLOS ONE},
  volume  = {12},
  number  = {7},
  pages   = {e0180944},
  year    = {2017},
  doi     = {10.1371/journal.pone.0180944}
}

@article{bailey2017probability,
  title   = {The Probability of Backtest Overfitting},
  author  = {Bailey, David H. and Borwein, Jonathan M. and L{\'o}pez de Prado, Marcos and Zhu, Qiji Jim},
  journal = {Journal of Computational Finance},
  volume  = {20},
  number  = {4},
  pages   = {39--69},
  year    = {2017},
  doi     = {10.21314/JCF.2016.322}
}

@inproceedings{abdulkadir2013unscented,
  title     = {Unscented Kalman Filter for Noisy Multivariate Financial Time-Series Data},
  author    = {Said Jadid Abdulkadir and Suet-Peng Yong},
  booktitle = {Multi-disciplinary Trends in Artificial Intelligence},
  series    = {Lecture Notes in Computer Science},
  volume    = {8271},
  pages     = {87--96},
  year      = {2013},
  publisher = {Springer},
  doi       = {10.1007/978-3-642-44949-9_9}
}

@article{alrumaih2002wavelet,
  title   = {Time Series Forecasting Using Wavelet Denoising: An Application to Saudi Stock Index},
  author  = {Rumaih M. Alrumaih and Mohammad A. Al-Fawzan},
  journal = {Journal of King Saud University: Engineering Sciences},
  volume  = {14},
  number  = {2},
  pages   = {221--233},
  year    = {2002},
  doi     = {10.1016/S1018-3639(18)30755-4}
}

@article{song2021forecasting,
  title={Forecasting Stock Market Indices Using Padding-Based Fourier Transform Denoising and Time Series Deep Learning Models},
  author={Song, Donghwan and Baek, Adrian Matias Chung and Kim, Namhun},
  journal={IEEE Access},
  volume={9},
  pages={83786--83796},
  year={2021},
  publisher={IEEE},
  doi={10.1109/ACCESS.2021.3086537}
}

@article{zhang2025major,
  title={Major Issues in High-Frequency Financial Data Analysis: A Survey of Solutions},
  author={Zhang, Lu and Hua, Lei},
  journal={Mathematics},
  volume={13},
  number={3},
  pages={347},
  year={2025},
  publisher={MDPI},
  doi={10.3390/math13030347}
}

@inproceedings{liu2020finrl,
  title={FinRL: A deep reinforcement learning library for automated stock trading in quantitative finance},
  author={Liu, Xiao-Yang and Yang, Hongyang and Chen, Qian and Zhang, Runjia and Yang, Liuqing and Xiao, Bowen and Wang, Christina Dan},
  booktitle={Deep RL Workshop, NeurIPS 2020},
  year={2020},
  url={https://arxiv.org/abs/2011.09607}
}

@inproceedings{liu2022finrlmeta,
  title={FinRL-Meta: Market Environments and Benchmarks for Data-Driven Financial Reinforcement Learning},
  author={Liu, Xiao-Yang and Xia, Ziyi and Rui, Jingyang and Gao, Jiechao and Yang, Hongyang and Zhu, Ming and Wang, Christina Dan and Wang, Zhaoran and Guo, Jian},
  booktitle={Advances in Neural Information Processing Systems 35 (NeurIPS 2022) Datasets and Benchmarks Track},
  volume={35},
  year={2022},
  url={https://papers.neurips.cc/paper_files/paper/2022/hash/0bf54b80686d2c4dc0808c2e98d430f7-Abstract-Datasets_and_Benchmarks.html}
}

@article{pippas2024evolution,
  title={The Evolution of Reinforcement Learning in Quantitative Finance: A Survey},
  author={Pippas, Nikolaos and Ludvig, Elliot A. and Turkay, Cagatay},
  journal={ACM Computing Surveys},
  volume={57},
  number={11},
  pages={1--51},
  year={2025},
  doi={10.1145/3733714}
}

@inproceedings{fujimoto2018addressing,
  title={Addressing function approximation error in actor-critic methods},
  author={Fujimoto, Scott and van Hoof, Herke and Meger, David},
  booktitle={International conference on machine learning},
  pages={1587--1596},
  year={2018},
  organization={PMLR},
  url={https://proceedings.mlr.press/v80/fujimoto18a.html}
}

@inproceedings{haarnoja2018soft,
  title={Soft actor-critic: Off-policy maximum entropy deep reinforcement learning with a stochastic actor},
  author={Haarnoja, Tuomas and Zhou, Aurick and Abbeel, Pieter and Levine, Sergey},
  booktitle={International conference on machine learning},
  pages={1861--1870},
  year={2018},
  organization={PMLR},
  url={https://proceedings.mlr.press/v80/haarnoja18b.html}
}

@inproceedings{chen2021randomized,
  title={Randomized ensembled double q-learning: Learning fast without a model},
  author={Chen, Xinyue and Wang, Che and Zhou, Zijian and Ross, Keith},
  booktitle={International Conference on Learning Representations},
  year={2021},
  url={https://openreview.net/forum?id=AY8zfZm0tDd}
}

@article{tran2026hybrid,
  title={Hybrid TD3: Overestimation Bias Analysis and Stable Policy Optimization for Hybrid Action Space},
  author={Tran, Thanh-Tuan and Canh, Thanh Nguyen and Chong, Nak Young and HoangVan, Xiem},
  journal={arXiv preprint arXiv:2603.01302},
  year={2026}
}

@article{mair2020robust,
  title={Robust statistical methods in R using the WRS2 package},
  author={Mair, Patrick and Wilcox, Rand},
  journal={Behavior research methods},
  volume={52},
  number={2},
  pages={464--488},
  year={2020},
  publisher={Springer},
  doi={10.3758/s13428-019-01246-w}
}

@inproceedings{lyle2024normalization,
  title={Normalization and effective learning rates in reinforcement learning},
  author={Lyle, Clare and Zheng, Zeyu and Khetarpal, Khimya and Martens, James and van Hasselt, Hado and Pascanu, Razvan and Dabney, Will},
  booktitle={Advances in Neural Information Processing Systems},
  volume={37},
  year={2024},
  url={https://papers.nips.cc/paper_files/paper/2024/hash/c04d37be05ba74419d2d5705972a9d64-Abstract-Conference.html}
}

@article{jiang2017deep,
  title={A deep reinforcement learning framework for the financial portfolio management problem},
  author={Jiang, Zhengyao and Xu, Dixing and Liang, Jinjun},
  journal={arXiv preprint arXiv:1706.10059},
  year={2017}
}

@article{zhang2019deep,
  title={Deep reinforcement learning for trading},
  author={Zhang, Zihao and Zohren, Stefan and Roberts, Stephen},
  journal={The Journal of Financial Data Science},
  volume={2},
  number={2},
  pages={25--40},
  year={2020},
  doi={10.3905/jfds.2020.1.030}
}

@article{acuto2022variational,
  title={Variational quantum soft actor-critic for robotic arm control},
  author={Acuto, Alberto and Barill{\`a}, Paola and Bozzolo, Ludovico and Conterno, Matteo and Pavese, Mattia and Policicchio, Antonio},
  journal={arXiv preprint arXiv:2212.11681},
  year={2022}
}

@article{lokossou2025quantum,
  title={Quantum deep reinforcement learning for humanoid robot navigation task},
  author={Lokossou, Romerik and Girma, Birhanu Shimelis and Tonguz, Ozan K. and Biyabani, Ahmed},
  journal={arXiv preprint arXiv:2509.11388},
  year={2025}
}

@inproceedings{liu2025quantum_fintech,
  title={Quantum-Enhanced Reinforcement Learning with {LSTM} Forecasting Signals for Optimizing Fintech Trading Decisions},
  author={Liu, Yen-Ku and Pan, Yun-Huei and Lu, Pei-Fan and Tsai, Yun-Cheng and Chen, Samuel Yen-Chi},
  booktitle={Proceedings of the IEEE International Conference on Quantum Computing and Engineering (QCE)},
  pages={235--240},
  year={2025},
  doi={10.1109/QCE65121.2025.10325}
}

@article{hohenfeld2024quantum,
  title={Quantum deep reinforcement learning for robot navigation tasks},
  author={Hohenfeld, Hans and Heimann, Dirk and Wiebe, Felix and Kirchner, Frank},
  journal={IEEE Access},
  volume={12},
  pages={87217--87236},
  year={2024},
  publisher={IEEE},
  doi={10.1109/ACCESS.2024.3417808}
}

@inproceedings{jerbi2023quantum,
  title={Quantum Policy Gradient Algorithms},
  author={Jerbi, Sofiene and Cornelissen, Arjan and Ozols, M{\=a}ris and Dunjko, Vedran},
  booktitle={18th Conference on the Theory of Quantum Computation, Communication and Cryptography (TQC 2023)},
  volume={266},
  pages={13:1--13:24},
  year={2023},
  organization={Schloss Dagstuhl-Leibniz-Zentrum f{\"u}r Informatik},
  doi={10.4230/LIPIcs.TQC.2023.13}
}

@inproceedings{xu2025accelerating,
  title={Accelerating Quantum Reinforcement Learning with a Quantum Natural Policy Gradient Based Approach},
  author={Xu, Yang and Aggarwal, Vaneet},
  booktitle={International Conference on Machine Learning},
  year={2025},
  url={https://proceedings.mlr.press/v267/xu25a.html}
}

@article{perez2020data,
  title={Data re-uploading for a universal quantum classifier},
  author={P{\'e}rez-Salinas, Adri{\'a}n and Cervera-Lierta, Alba and Gil-Fuster, Elies and Latorre, Jos{\'e} I.},
  journal={Quantum},
  volume={4},
  pages={226},
  year={2020},
  publisher={Verein zur F{\"o}rderung des Open Access Publizierens in den Quantenwissenschaften},
  doi={10.22331/q-2020-02-06-226}
}

@inproceedings{martyniuk2024quantum,
  title={Quantum Architecture Search: A Survey},
  author={Martyniuk, Darya and Jung, Johannes and Paschke, Adrian},
  booktitle={Proceedings of the IEEE International Conference on Quantum Computing and Engineering (QCE)},
  pages={1695--1706},
  year={2024},
  doi={10.1109/QCE60285.2024.00198}
}

@article{caro2021encoding,
  title={Encoding-dependent generalization bounds for parametrized quantum circuits},
  author={Caro, Matthias C. and Gil-Fuster, Elies and Meyer, Johannes Jakob and Eisert, Jens and Sweke, Ryan},
  journal={Quantum},
  volume={5},
  pages={582},
  year={2021},
  publisher={Verein zur F{\"o}rderung des Open Access Publizierens in den Quantenwissenschaften},
  doi={10.22331/q-2021-11-17-582}
}

@article{roseler2026how,
  title={How to find expressible and trainable parameterized quantum circuits?},
  author={R{\"o}seler, Peter and Willsch, Dennis and Michielsen, Kristel},
  journal={arXiv preprint arXiv:2603.14451},
  year={2026}
}

@article{sweke2025potential,
  title={Potential and limitations of random Fourier features for dequantizing quantum machine learning},
  author={Sweke, Ryan and Recio-Armengol, Erik and Jerbi, Sofiene and Gil-Fuster, Elies and Fuller, Bryce and Eisert, Jens and Meyer, Johannes Jakob},
  journal={Quantum},
  volume={9},
  pages={1640},
  year={2025},
  publisher={Verein zur F{\"o}rderung des Open Access Publizierens in den Quantenwissenschaften},
  doi={10.22331/q-2025-02-20-1640}
}

@article{cameron2021robust,
  title={Robust and scalable sde learning: A functional perspective},
  author={Cameron, Scott and Cameron, Tyron and Pretorius, Arnu and Roberts, Stephen},
  journal={arXiv preprint arXiv:2110.05167},
  year={2021}
}

@article{duarte2024machine,
  title={Machine Learning for Continuous-Time Finance},
  author={Duarte, Victor and Duarte, Diogo and Silva, Dejanir H.},
  journal={The Review of Financial Studies},
  volume={37},
  number={11},
  pages={3217--3271},
  year={2024},
  doi={10.1093/rfs/hhae043}
}

@article{schulman2017proximal,
  title={Proximal policy optimization algorithms},
  author={Schulman, John and Wolski, Filip and Dhariwal, Prafulla and Radford, Alec and Klimov, Oleg},
  journal={arXiv preprint arXiv:1707.06347},
  year={2017}
}

@inproceedings{mnih2016asynchronous,
  title={Asynchronous methods for deep reinforcement learning},
  author={Mnih, Volodymyr and Badia, Adria Puigdomenech and Mirza, Mehdi and Graves, Alex and Lillicrap, Timothy and Harley, Tim and Silver, David and Kavukcuoglu, Koray},
  booktitle={International conference on machine learning},
  pages={1928--1937},
  year={2016},
  organization={PMLR},
  url={https://proceedings.mlr.press/v48/mniha16.html}
}

@inproceedings{kuznetsov2020controlling,
  title     = {Controlling Overestimation Bias with Truncated Mixture of Continuous Distributional Quantile Critics},
  author    = {Kuznetsov, Arsenii and Shvechikov, Pavel and Grishin, Alexander and Vetrov, Dmitry},
  booktitle = {Proceedings of the 37th International Conference on Machine Learning},
  series    = {Proceedings of Machine Learning Research},
  volume    = {119},
  pages     = {5556--5566},
  year      = {2020},
  publisher = {PMLR},
  doi       = {10.48550/arXiv.2005.04269}
}

@article{pearson1901lines,
  title={LIII. On lines and planes of closest fit to systems of points in space},
  author={Pearson, Karl},
  journal={The London, Edinburgh, and Dublin philosophical magazine and journal of science},
  volume={2},
  number={11},
  pages={559--572},
  year={1901},
  publisher={Taylor \& Francis},
  doi={10.1080/14786440109462720}
}

@article{hornik1989multilayer,
  title={Multilayer feedforward networks are universal approximators},
  author={Hornik, Kurt and Stinchcombe, Maxwell and White, Halbert},
  journal={Neural networks},
  volume={2},
  number={5},
  pages={359--366},
  year={1989},
  publisher={Elsevier},
  doi={10.1016/0893-6080(89)90020-8}
}

@article{hinton2006reducing,
  title={Reducing the dimensionality of data with neural networks},
  author={Hinton, Geoffrey E and Salakhutdinov, Ruslan R},
  journal={science},
  volume={313},
  number={5786},
  pages={504--507},
  year={2006},
  publisher={American Association for the Advancement of Science},
  doi={10.1126/science.1127647}
}

@article{krogh1991simple,
  title={A simple weight decay can improve generalization},
  author={Krogh, Anders and Hertz, John},
  journal={Advances in neural information processing systems},
  volume={4},
  year={1991},
  url={https://proceedings.neurips.cc/paper/1991/hash/8eefcfdf5990e441f0fb6f3fad709e21-Abstract.html}
}

@article{ba2016layer,
  title={Layer normalization},
  author={Ba, Jimmy Lei and Kiros, Jamie Ryan and Hinton, Geoffrey E},
  journal={arXiv preprint arXiv:1607.06450},
  year={2016}
}

@article{miyato2018spectral,
  title={Spectral normalization for generative adversarial networks},
  author={Miyato, Takeru and Kataoka, Toshiki and Koyama, Masanori and Yoshida, Yuichi},
  journal={arXiv preprint arXiv:1802.05957},
  year={2018}
}

@article{choi2018pact,
  title={Pact: Parameterized clipping activation for quantized neural networks},
  author={Choi, Jungwook and Wang, Zhuo and Venkataramani, Swagath and Chuang, Pierce I-Jen and Srinivasan, Vijayalakshmi and Gopalakrishnan, Kailash},
  journal={arXiv preprint arXiv:1805.06085},
  year={2018}
}

@inproceedings{xu2017feature,
  title={Feature Squeezing: Detecting Adversarial Examples in Deep Neural Networks},
  author={Xu, Weilin and Evans, David and Qi, Yanjun},
  booktitle={Network and Distributed System Security Symposium (NDSS)},
  year={2018},
  doi={10.14722/ndss.2018.23198},
  url={https://www.ndss-symposium.org/wp-content/uploads/2018/02/ndss2018_03A-4_Xu_paper.pdf}
}
% ------------------------------------------------------------------------
% Appendix
% --------------------------------------------------------------------
\appendix
% -----------------------------------------------------------
\section{Representation Quality Evaluation}
\label{appendix:encoder_diagnostics}
To verify whether PQC fulfills its core design goal of noise suppression and signal structure preservation, we analyze the variance and VAF of the learned representations. Under identical experimental settings, we extract representations from the same raw input state and report two groups of comparisons: representative spectral/filtering examples, including FPQC-SAC, RFF/Fourier features~\citep{sweke2025potential}, Wavelet filtering~\citep{alrumaih2002wavelet}, and Kalman filtering~\citep{abdulkadir2013unscented}, and the classical bottleneck controls used in Table~\ref{tab:bottleneck_comparison}. Total Latent Variance quantifies the overall intensity of spatial noise perturbations in the representation manifold, while VAF@1 (Variance Accounted For by the 1st PC) assesses structural coherence and the representation's capability to linearly preserve predictable environmental structure, with higher values implying fewer unstructured noise components.

\begin{figure}[htbp]
  \centering
  \includegraphics[width=\textwidth]{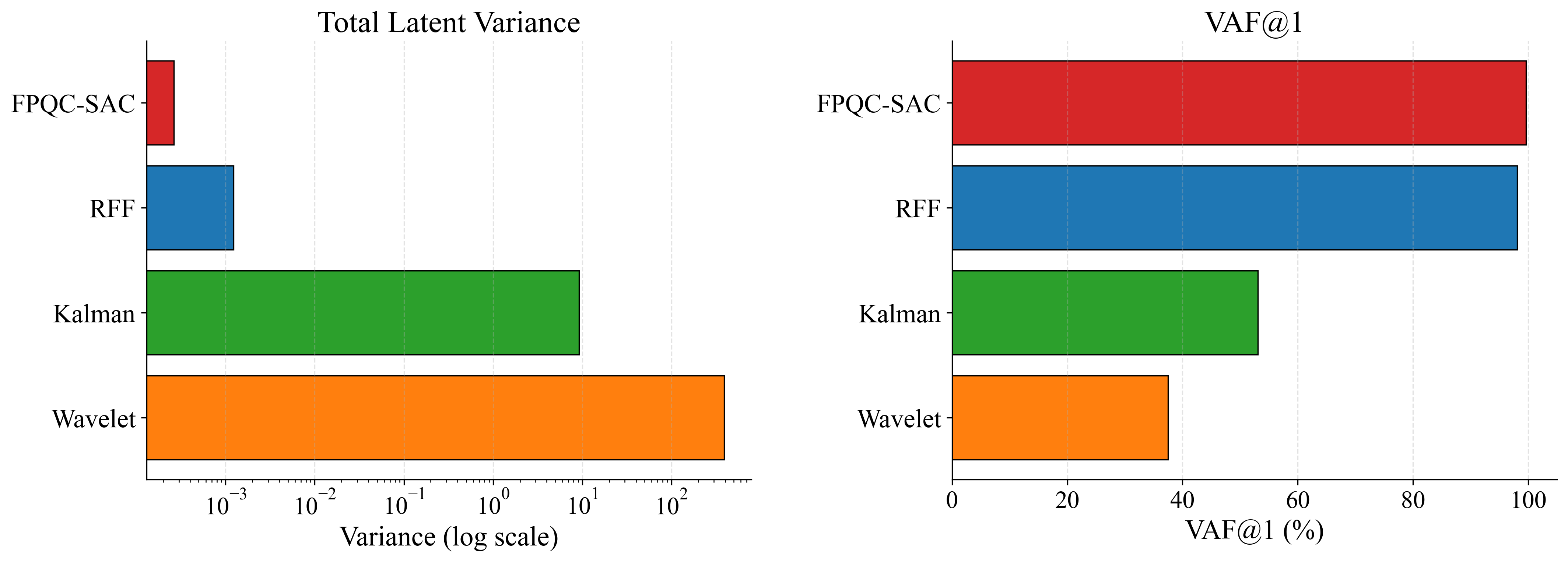}
  \caption{Representation quality for FPQC-SAC, RFF/Fourier features, Wavelet filtering, and Kalman filtering. Lower Total Latent Variance and higher VAF@1 indicate stronger noise suppression and structural coherence.}
  \label{fig:appendix_representation_vaf}
\end{figure}

\begin{figure}[htbp]
  \centering
  \includegraphics[width=\textwidth]{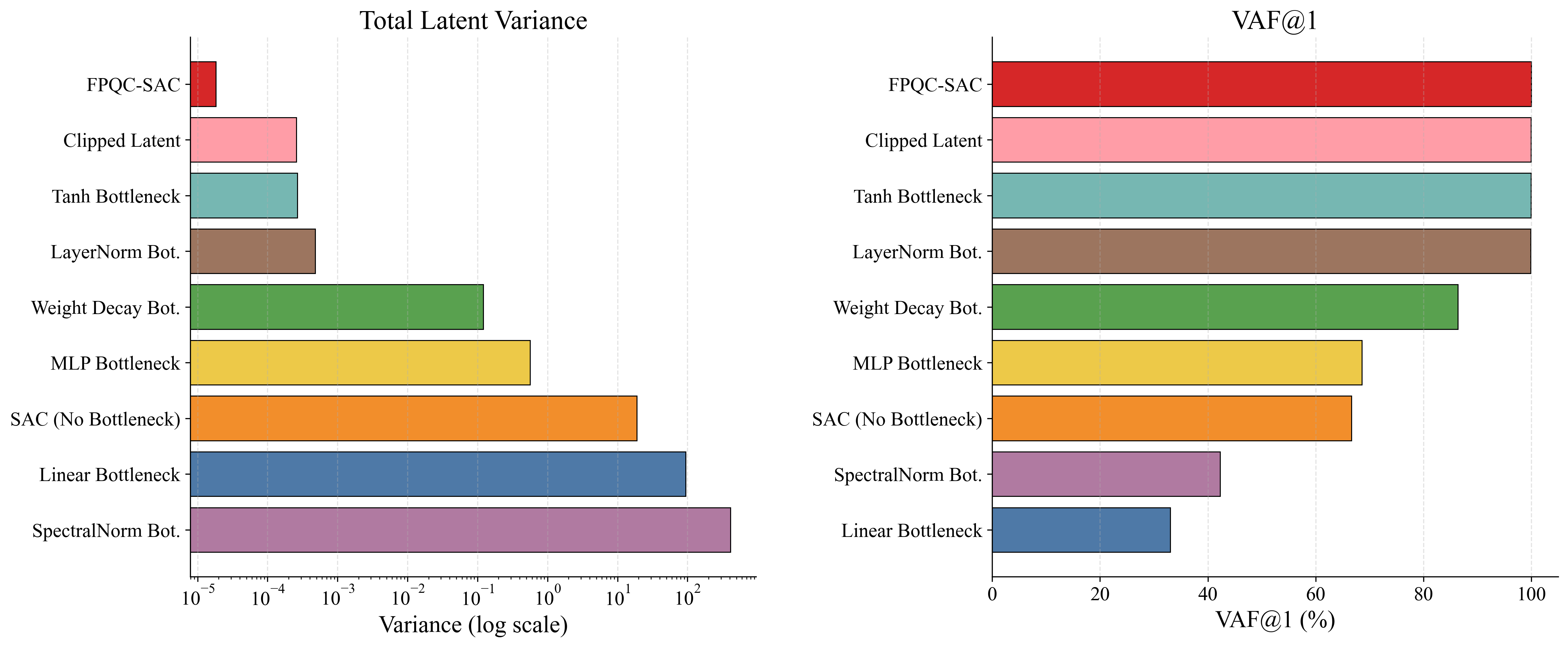}
  \caption{Representation quality for the classical bottleneck controls in Table~\ref{tab:bottleneck_comparison}. Lower Total Latent Variance and higher VAF@1 indicate stronger noise suppression and structural coherence.}
  \label{fig:appendix_bottleneck_representation_vaf}
\end{figure}

\begin{table}[htbp]
  \centering
  \caption{Performance and available Q-value training diagnostics for FPQC-SAC and representative signal-processing encoders.}
  \label{tab:appendix_spectral_diagnostics}
  \small
  \setlength{\tabcolsep}{2.5pt}
  \resizebox{\textwidth}{!}{
  \begin{tabular}{lcccccccc}
    \toprule
    \textbf{Encoder} & \textbf{CR\%$\uparrow$} & \textbf{AR\%$\uparrow$} & \textbf{SR$\uparrow$} & \textbf{Sortino$\uparrow$} & \textbf{Calmar$\uparrow$} & \textbf{Q-gap$\downarrow$} & \textbf{TD-var$\downarrow$} & \textbf{C-loss-var$\downarrow$} \\
    \midrule
    \textbf{FPQC-SAC} & \textbf{195.32$\pm$26.94} & \textbf{50.05$\pm$5.15} & \textbf{1.71$\pm$0.15} & \textbf{2.13$\pm$0.19} & \textbf{1.75$\pm$0.20} & \textbf{$4.55{\times}10^{-4}\pm4.61{\times}10^{-2}$} & \textbf{0.535$\pm$0.69} & \textbf{8.19$\pm$14} \\
    RFF & 162.23$\pm$37.25 & 43.30$\pm$7.76 & 1.53$\pm$0.22 & 1.94$\pm$0.25 & 1.61$\pm$0.29 & $5.79{\times}10^{-4}\pm5.10{\times}10^{-2}$ & 0.552$\pm$0.67 & 8.08$\pm$12 \\
    Wavelet & 138.44$\pm$36.39 & 38.23$\pm$8.00 & 1.49$\pm$0.17 & 1.76$\pm$0.27 & 1.32$\pm$0.35 & -0.252$\pm$0.37 & $7.81{\times}10^2\pm1.74{\times}10^3$ & $2.61{\times}10^3\pm4.96{\times}10^3$ \\
    Kalman & 130.44$\pm$23.78 & 36.66$\pm$5.26 & 1.50$\pm$0.15 & 1.72$\pm$0.23 & 1.23$\pm$0.26 & -0.170$\pm$0.05 & 10.00$\pm$7.62 & 61.09$\pm$95.59 \\
    \bottomrule
  \end{tabular}
  }
\end{table}

\FloatBarrier

\section{Reproducibility Details}
\label{appendix:repro_details}
\subsection{Markov Decision Process (MDP) and Environment Settings}

The trading environment is formulated as a continuous-control Markov Decision Process (MDP). The portfolio consists of 6 highly volatile U.S. assets: AAPL, AMZN, GOOGL, MSFT, QQQ, and SPY.
\begin{table}[htbp]
  \centering
  \caption{State space, action space, and environment configuration.}
  \label{tab:mdp_env_setup}
  \small
  \begin{tabular}{@{}p{0.27\textwidth}p{0.66\textwidth}@{}}
    \toprule
    \textbf{Parameter} & \textbf{Value / Description} \\
    \midrule
    \textbf{State Space Dimension} & $64$ \\
    State Composition & $[1 \text{ Cash}, 6 \text{ Holdings}, 6 \text{ Prices}, 8 \times 6 \text{ Tech Indicators}, 3 \text{ Additional Market Features}]$ \\
    Technical Indicators (8) & MACD, BOLL\_UB, BOLL\_LB, RSI\_30, CCI\_30, DX\_30, SMA\_30, SMA\_60 \\
    State Normalization & None (Raw values are used, with internal cash scaling) \\
    \midrule
    \textbf{Action Space} & Continuous Box space, Dimension $6$ \\
    Action Meaning & Values mapped to $[-1, 1]$. Positive = Buy, Negative = Sell \\
    Maximum Shares ($h_{max}$) & 100 shares per transaction \\
    \midrule
    \multicolumn{2}{@{}l}{\textbf{Trading Parameters}} \\
    Initial Account Balance & \$100,000 \\
    Transaction Fee Rate & 0.01\% ($0.0001$) for both buy and sell orders \\
    Reward Function & 
    \begin{minipage}[t]{\linewidth}
    Stepwise change in portfolio value plus a 10-day look-back Sharpe reward:
    \begin{equation}
    \label{app:eq:reward_function}
    \begin{aligned}
    R_t &= r_t^\text{step} + r_t^\text{lookback}, \\
    r_t^\text{step} &= A_t - A_{t-1}, \\
    r_t^\text{lookback} &=
    \frac{\bar{\Delta A}_{t-10:t}}{\sqrt{\mathrm{Var}(\Delta A_{t-10:t}) + \epsilon}}, \\
    \bar{\Delta A}_{t-10:t} &=
    \frac{1}{10} \sum_{\tau=t-9}^{t} \Delta A_\tau, \quad
    \Delta A_\tau = A_\tau - A_{\tau-1}.
    \end{aligned}
    \end{equation}
    where $\epsilon > 0$ is a small constant to avoid division by zero.
    \end{minipage} \\
    Reward Scaling & $1 \times 10^{-4}$ \\
    \bottomrule
  \end{tabular}
\end{table}
In implementation, this look-back reward adjustment helps the agent consider short-term risk-adjusted performance in addition to immediate asset gains.
\subsection{Quantum Circuit Architecture and Network Topology}

In our FPQC-SAC implementation, the quantum component acts specifically as a \textit{Quantum Feature Extractor} that processes the raw market state before passing the representations to the classical Actor and Critic heads. 

The topology follows a bottleneck structure:
\begin{enumerate}
    \item \textbf{Pre-Net (Classical):} A linear projection with ReLU activation maps the 64-dimensional state to 7 dimensions ($64 \to 7$).
    \item \textbf{PQC Layer (Quantum):} The 7-dimensional vector is encoded into 7 qubits via Angle Embedding ($R_x$ rotations). This is followed by 2 layers of \texttt{BasicEntanglerLayers} (parameterized rotations connected by circular CNOTs). The output is obtained via Pauli-Z expectation values ($\langle Z \rangle$) across all 7 qubits.
    \item \textbf{Post-Net (Classical):} A linear projection with ReLU activation expands the 7-dimensional quantum measurement back to a 64-dimensional feature vector ($7 \to 64$).
    \item \textbf{Actor-Critic Heads:} Standard SB3 MLP networks (two layers of 256 units: $[256, 256]$) process the 64-dimensional feature vector to output policies and Q-values
\end{enumerate}

\begin{table}[htbp]
  \centering
  \caption{Quantum Parameterized Quantum Circuit (PQC) settings.}
  \label{tab:quantum_setup}
  \small
  \begin{tabular}{ll}
    \toprule
    \textbf{Parameter} & \textbf{Configuration} \\
    \midrule
    Number of Qubits ($N$) & 7 \\
    Encoding Strategy & \texttt{qml.AngleEmbedding} \\
    Variational Layers & 2 layers of \texttt{qml.BasicEntanglerLayers} \\
    Observable & Expectation value of Pauli-Z ($\langle Z \rangle$) per qubit \\
    Simulator Backend & \texttt{default.qubit} (PennyLane) \\
    Differentiation Method & Exact backpropagation (\texttt{diff\_method="backprop"}) \\
    Total PQC Trainable Weights & 14 parameters (Shape: $2 \times 7$) \\
    \bottomrule
  \end{tabular}
\end{table}

\subsection{Training Hyperparameters}

The models are trained over the period from \textbf{2013-01-02 to 2018-12-31} and strictly evaluated out-of-sample from \textbf{2019-01-02 to 2021-08-31} ($\sim$672 trading days). Each agent is trained for 25,000 environment steps. All DRL experiments are executed across 20 distinct, fixed random seeds to ensure statistical reliability. 

\begin{table}[htbp]
  \centering
  \caption{Algorithmic hyperparameters for standard and quantum DRL baselines.}
  \label{tab:hyperparameters}
  \small
  \begin{tabular}{ll}
    \toprule
    \textbf{Hyperparameter} & \textbf{Value} \\
    \midrule
    Total Environment Steps & 25,000 \\
    Learning Rate & $3 \times 10^{-4}$ \\
    Batch Size (SAC, DDPG, TD3) & 64 \\
    Replay Buffer Size & 200,000 \\
    Discount Factor ($\gamma$) & 0.99 \\
    Target Smoothing Coefficient ($\tau$) & 0.005 \\
    Entropy Coefficient (SAC/FPQC-SAC) & ``auto'' (dynamically adjusted to target entropy $-6$) \\
    Optimizer & Adam \\
    \bottomrule
  \end{tabular}
\end{table}

% ========================================================================
% Extended Theoretical Analysis and Proofs
% ========================================================================

\section{Extended Theoretical Analysis and Proofs}
\label{app:extended_theory}

\subsection{Proof of Proposition 1: Expectation Attenuation}
\label{app:proof_prop1}

\begin{proof}
Given the effective latent SDE $dZ_s = \mu_Z ds + \sigma_Z dW_s$ (from Lemma \ref{lemma:prenet_jacobian}) over $s \in [t, t+\Delta t]$. The quantum angle encoding maps this state into complex periodic components. Let $v_\omega(Z_s) = \exp(i\omega^\top Z_s)$. Applying It\^o's formula to the complex-valued process $v_\omega$:
\begin{equation}
\label{app:eq:proof_prop1_ito}
dv_\omega(Z_s) = v_\omega(Z_s) \left( i\omega^\top \mu_Z - \frac{1}{2} \omega^\top \Sigma_Z \omega \right) ds + v_\omega(Z_s) i\omega^\top \sigma_Z dW_s
\end{equation}
Taking the expectation $\mathbb{E}[\cdot]$ over local Brownian paths, the It\^o integral with respect to Brownian motion is a martingale and evaluates to zero. Solving the deterministic ordinary differential equation over $\Delta t$ yields:
\begin{equation}
\label{app:eq:proof_prop1_expectation}
    \mathbb{E}[e^{i\omega^\top Z_{t+\Delta t}}] = e^{i\omega^\top Z_t} \exp\left(i \int_t^{t+\Delta t} \omega^\top \mu_Z ds \right) \exp\left(-\frac{1}{2} \int_t^{t+\Delta t} \omega^\top J_h \Sigma_X J_h^\top \omega ds \right)
\end{equation}
This establishes that despite the unconstrained affine mapping $J_h$, the complex phase mathematically enforces an expectation-level exponential attenuation of the diffusion variance.
\end{proof}

This section rigorously formalizes the theoretical framework established in Section~\ref{sec:methodology}, providing end-to-end mathematical proofs connecting the financial It\^o diffusion, the entropy trap, and the quantum representation bounds.
\subsection{Formalizing the Financial Entropy Trap}
\label{app:entropy_trap}

When Q-values are perturbed by environmental noise, the soft value update operator of SAC, $\alpha \log \int \exp(Q/\alpha) da$, systematically converts zero-mean noise into a positive bias. Inspired by the analysis of the log-sum-exp overestimation tendency by Kuznetsov et al. \cite{kuznetsov2020controlling}, we contextualize this mechanism within the bootstrapped financial low-SNR setting, identifying it as the structural root of SAC's systematic failure in such environments.

Let the true action-value be $Q^\star(s, a)$. Due to environmental diffusion, the estimated Q-value contains a zero-mean perturbation: $\tilde{Q}(s, a) = Q^\star(s, a) + \epsilon(s, a)$, with $\mathbb{E}_\epsilon[\epsilon(s, a)] = 0$. The soft value estimate over the continuous action space $\mathcal{A}$ is given by $V(s) = \alpha \log \int_{\mathcal{A}} \exp(\tilde{Q}(s, a)/\alpha) da$. 

Because the continuous $\log \int \exp(\cdot)$ operator is strictly convex, applying Jensen's inequality with respect to the noise distribution yields a strict positive bias:
\begin{equation}
\label{app:eq:entropy_bias}
\begin{aligned}
    \mathbb{E}_\epsilon[V(s)] &= \mathbb{E}_\epsilon \left[ \alpha \log \int_{\mathcal{A}} \exp\left(\frac{Q^\star(s, a) + \epsilon(s, a)}{\alpha}\right) da \right] \\
    &> \alpha \log \int_{\mathcal{A}} \exp\left(\frac{\mathbb{E}_\epsilon[Q^\star(s, a) + \epsilon(s, a)]}{\alpha}\right) da \\
    &= \alpha \log \int_{\mathcal{A}} \exp\left(\frac{Q^\star(s, a)}{\alpha}\right) da \equiv V^\star(s).
\end{aligned}
\end{equation}

Furthermore, assuming the environmental noise $\epsilon(s, a)$ is locally sub-Gaussian with a variance proxy $\sigma^2_\epsilon$, the moment-generating function restricts this overestimation bias to a deterministic upper bound :
\begin{equation}
\label{app:eq:entropy_bias_bound}
    0 < \mathbb{E}_\epsilon[V(s)] - V^\star(s) \le \frac{\sigma^2_\epsilon}{2\alpha}.
\end{equation}

\textbf{Mathematical Implication:} These inequalities establish a harsh reality: \textit{even if the noise is strictly zero-mean, the expected value estimation is deterministically skewed upward.} Within SAC's bootstrapping iterations, this positive bias is repeatedly injected into the target network and continuously amplified, forming the mathematical genesis of the Financial Entropy Trap. The upper bound $\mathcal{O}(\sigma^2_\epsilon / \alpha)$ explicitly dictates that the larger the market diffusion variance $\sigma^2_\epsilon$, the more severe the overestimation; furthermore, as the entropy temperature $\alpha$ decays during training, this explosive bias is further exacerbated.

\textbf{Financial Intuition:} The consequence of this mathematical mechanism can be intuitively understood through a notorious market phenomenon. In a highly uncertain environment, asset price fluctuations inherently lack clear directional signals. However, some traders may misinterpret a random upward noise spike as a genuine trend signal and choose to follow. These follow-up actions artificially drive the price higher, attracting even more participants and forming a self-reinforcing positive feedback loop. 

Yet, when the price inevitably reverts to its fundamentals, the accumulated directional positions reverse almost synchronously. The resulting stampede causes a drawdown far exceeding the initial artificial gains. Late participants not only fail to realize their paper profits but suffer severe real losses. This market dynamic perfectly mirrors the catastrophic policy degradation of an RL agent falling into the Financial Entropy Trap.

\subsection{Pre-Net Jacobian and Local Effective It\^o Diffusion}
\label{lemma:prenet_jacobian}

We use a standard local reduction to connect the raw It\^o state process to the latent diffusion analyzed in Appendix~\ref{app:proof_prop1}. Prior work in continuous-time finance has established the use of neural-network Jacobians and automatic differentiation together with It\^o's lemma for high-dimensional stochastic dynamic systems \citep{duarte2024machine}. Separately, SDE-learning literature has observed that when the transformation network is piecewise linear, such as ReLU networks, the second-order
It\^o correction vanishes almost everywhere, reducing the transformed dynamics to Jacobian-vector products \citep{cameron2021robust}.

Concretely, let the raw state follow
\begin{equation}
\label{app:eq:raw_sde}
dX_t = \mu_X dt + \sigma_X dW_t,
\end{equation}
and let the preprocessed latent state be
\begin{equation}
\label{app:eq:latent_state}
Z_t = h_{\mathrm{pre}}(X_t).
\end{equation}
Since the Pre-Net uses piecewise-linear activations such as ReLU, it is locally affine within each fixed activation region. Away from nonsmooth activation boundaries, the diffusion component is therefore locally propagated through the Pre-Net Jacobian:
\begin{equation}
\label{app:eq:latent_sde}
dZ_t
\approx
J_h(X_t)\mu_X dt + J_h(X_t)\sigma_X dW_t,
\end{equation}
where $J_h(X_t)$ denotes the local Jacobian of $h_{\mathrm{pre}}$ at $X_t$. Consequently,
the local instantaneous covariance of the latent process is
\begin{equation}
\label{app:eq:latent_cov}
\Sigma_Z(X_t)
=
\big(J_h(X_t)\sigma_X\big)\big(J_h(X_t)\sigma_X\big)^\top
=
J_h(X_t)\Sigma_XJ_h(X_t)^\top .
\end{equation}
This step is a standard local piecewise-linear approximation used to identify the effective diffusion covariance entering the PQC encoder. A fully exact treatment of nonsmooth activation boundaries would require generalized It\^o--Tanaka arguments and may introduce local-time correction terms at kink boundaries. Since our subsequent analysis only requires the local covariance propagated into the latent space, we use
$\Sigma_Z = J_h\Sigma_XJ_h^\top$
as the effective latent diffusion covariance.

\subsection{The Role of Entanglement: Cross-Dimensional Feature Crossing and Constructive Interference of Consensus Signals}
\label{app:entanglement_modes}

A Parameterized Quantum Circuit (PQC) without CNOT gates restricts Pauli-Z measurements strictly to separable single-coordinate Fourier modes. In financial markets, this is equivalent to relying on a single technical indicator, which is often noisy and prone to failure. In contrast, CNOT entanglement forces local single-qubit states to bind into multi-body Pauli strings, thereby generating inseparable cross-coordinate feature interactions. This entangled structure acts as an efficient high-dimensional feature crossing mechanism, enabling the model to evaluate joint signals across assets or indicators: only when multiple weak signals form a consistent structural consensus within the entangled state are they amplified through measurement (constructive interference); meanwhile, isolated, uncorrelated random noise cancels out in the complex phase superposition (destructive interference).

To illustrate this mechanism structurally, consider that in a circuit without entanglement, the unitary evolution factorizes completely across qubits:
\begin{equation}
\label{app:eq:separable_unitary}
    U(Z)=\bigotimes_{j=1}^n U_j(Z_j).
\end{equation}
For a local Pauli-Z measurement on qubit $j$, the expectation takes the form
\begin{equation}
\label{app:eq:single_qubit_measurement}
    f_j(Z) = \langle 0| U_j(Z_j)^\dagger Z U_j(Z_j) |0\rangle,
\end{equation}
and therefore depends only on the single latent coordinate $Z_j$. Under angle encoding, its Fourier spectrum is confined to single-coordinate modes such as $\omega=\pm e_j$. From a quantitative trading perspective, this corresponds to the model making decisions based solely on a single isolated indicator (such as an individual MACD or RSI), which frequently fails in low-SNR markets. By Proposition~\ref{prop:ito_attenuation}, these modes undergo expectation attenuation governed by:
\begin{equation}
\label{app:eq:single_mode_attenuation}
    \exp\left( - \frac{1}{2} \int_t^{t+\Delta t} e_j^\top\Sigma_Z e_j ds \right).
\end{equation}
This implies that without entanglement, phase attenuation occurs strictly independently along individual coordinate axes, and traditional linear feature extraction struggles to capture high-order correlations.

Now consider a CNOT gate with control qubit $j$ and target qubit $k$. In the Heisenberg picture:
\begin{equation}
\label{app:eq:cnot_heisenberg}
    \mathrm{CNOT}_{j\to k}^{\dagger} (I\otimes Z_k) \mathrm{CNOT}_{j\to k} = Z_j\otimes Z_k.
\end{equation}
Thus, a local measurement after the entangling layer can correspond to a two-body Pauli string before the entangler. Its expectation contains products of trigonometric phase terms, for example:
\begin{equation}
\label{app:eq:trig_product}
    \cos(Z_j)\cos(Z_k) = \frac{1}{2} \left[ \cos(Z_j+Z_k) + \cos(Z_j-Z_k) \right].
\end{equation}
Through this mechanism, entanglement naturally introduces joint Fourier modes:
\begin{equation}
\label{app:eq:joint_fourier_modes}
    \omega=e_j+e_k, \qquad \omega=e_j-e_k.
\end{equation}
These modes also inherit the attenuation law in Proposition~\ref{prop:ito_attenuation}:
\begin{equation}
\label{app:eq:joint_mode_attenuation}
    \exp\left( - \frac{1}{2} \int_t^{t+\Delta t} (e_j\pm e_k)^\top \Sigma_Z (e_j\pm e_k) ds \right).
\end{equation}
Expanding the quadratic form gives:
\begin{equation}
\label{app:eq:quadratic_form}
    (e_j\pm e_k)^\top \Sigma_Z (e_j\pm e_k) = (\Sigma_Z)_{jj} + (\Sigma_Z)_{kk} \pm 2(\Sigma_Z)_{jk}.
\end{equation}

This demonstrates that entanglement does not merely blindly add more non-linear features. It fundamentally provides an efficient cross-dimensional feature crossing mechanism. When processing financial features, this mechanism exhibits distinctly different filtering properties for "noise" versus "signal":

\textbf{1. Destructive Interference of Isolated Noise (Analogous to Denoising):} 
For independent, isotropic noise lacking intrinsic correlation (where $\Sigma_Z=\sigma^2 I$), the joint mode satisfies:
\begin{equation}
\label{app:eq:isotropic_joint_variance}
    (e_j\pm e_k)^\top\Sigma_Z(e_j\pm e_k) = 2\sigma^2,
\end{equation}
which is larger than the single-coordinate variance $\sigma^2$. This indicates that unstructured random noise lacking correlation will induce more violent phase fluctuations within these entangled cross modes, thereby accelerating attenuation and canceling each other out in the final expectation projection.

\textbf{2. Constructive Interference of Structural Consensus (Analogous to Signal Amplification):}
Conversely, effective trading signals are often hidden within the joint distributions or non-linear combinations of multiple indicators (e.g., price breakouts accompanied by specific volume and volatility contractions). When there is a weak drift structure shared across adjacent latent coordinates (i.e., a genuine market consensus exists among indicators), it induces coherent phase translations in the related qubits. Trainable rotation gates can dynamically adjust the measurement basis so that these structural consensus signals (correlated phase translations) are precisely projected onto selected Pauli readout directions. Along these directions, the phase contributions of multiple weak signals undergo constructive accumulation (constructive interference) and are effectively amplified.

In conclusion, entanglement does not clean up noise in a perfect mathematical sense. Instead, it forces independent indicators to coalesce into multi-body states. This setup not only uses the diffusion attenuation mechanism from Proposition~\ref{prop:ito_attenuation} to keep uncorrelated noise in check, but also adds the ability to capture---and amplify---the structural consensus embedded in financial signals.

\section{Weak-Signal-Preserving Property}
\label{app:snr_bound}
\begin{lemma}[Representation Output SNR Lower Bound]
Consider an environment transition $X \rightarrow X+\delta+\xi$, where
$\delta$ is a weak deterministic drift signal and $\xi$ is high-variance
zero-mean noise. Let
\begin{equation}
\label{app:eq:effective_jacobian}
J_{\phi}^{\mathrm{eff}}(X)
=
\nabla_X \mathbb{E}_{\xi}[\phi(X+\xi)]
\end{equation}
denote the effective local Jacobian of the expectation-level PQC mapping.
If the trained circuit is non-singular along the signal direction such that
\begin{equation}
\label{app:eq:snr_nonsingular}
\sigma_{\min}\!\left(
J_{\phi}^{\mathrm{eff}}\big|_{\mathrm{span}(\delta)}
\right)
\ge c > 0,
\end{equation}
then for any fixed nonzero drift $\delta$, the output signal-to-noise ratio
is lower-bounded by
\begin{equation}
\label{app:eq:snr_lower_bound}
\mathrm{SNR}_{\mathrm{out}}
\ge
\frac{c^2\|\delta\|_2^2}{k}.
\end{equation}
\end{lemma}

\begin{proof}
By a first-order Taylor expansion of the expected representation mapping,
the preserved signal magnitude satisfies
\begin{equation}
\label{app:eq:signal_taylor}
\begin{aligned}
\|\Delta \mathbb{E}[\phi]\|_2^2
=
\left\|
\mathbb{E}[\phi(X+\delta+\xi)]
-
\mathbb{E}[\phi(X+\xi)]
\right\|_2^2
\approx
\|J_{\phi}^{\mathrm{eff}}(X)\delta\|_2^2
\ge
c^2\|\delta\|_2^2.
\end{aligned}
\end{equation}

The output noise power is measured by the trace of the covariance across
the representation dimensions:
\begin{equation}
\label{app:eq:noise_power}
\operatorname{Tr}(\operatorname{Cov}[\phi(X+\xi)])
=
\sum_{j=1}^{k}
\operatorname{Var}(\phi_j(X+\xi)).
\end{equation}
By the bounded Pauli readout established in Proposition~\ref{prop:interference_attenuation}, each PQC feature
dimension satisfies $\phi_j(X+\xi)\in[-1,1]$. Therefore, by Popoviciu's
inequality,
\begin{equation}
\label{app:eq:popoviciu_bound}
\operatorname{Var}(\phi_j(X+\xi))
\le
\frac{(1-(-1))^2}{4}
=
1.
\end{equation}
Consequently,
\begin{equation}
\label{app:eq:noise_power_bound}
\operatorname{Tr}(\operatorname{Cov}[\phi(X+\xi)]) \le k.
\end{equation}
The output SNR therefore satisfies
\begin{equation}
\label{app:eq:snr_out_bound}
\begin{aligned}
\mathrm{SNR}_{\mathrm{out}}
=
\frac{\|\Delta \mathbb{E}[\phi]\|_2^2}
{\operatorname{Tr}(\operatorname{Cov}[\phi])}
\ge
\frac{c^2\|\delta\|_2^2}{k}.
\end{aligned}
\end{equation}
This shows that the PQC bottleneck can preserve a nonzero deterministic
signal component along learned non-singular directions, while imposing a
dimension-dependent ceiling on the output noise power.
\end{proof}
\section{Extended Out-of-Sample Extrapolation (2019--2024)}
\label{appendix:diagnostics}

To further evaluate the long-term robustness and generalization capability of the learned policies, we conduct an extended out-of-sample extrapolation experiment. In this setting, the DRL agents, which were trained solely on data up to 2018-12-31, are deployed without any retraining over an extended testing horizon from 2019-01-02 to 2023-12-31. This prolonged period introduces entirely unseen macroeconomic regimes, most notably the sustained market downturn of 2022 and the subsequent recovery phase in 2023.
\begin{figure}[htbp]
  \centering
  \includegraphics[width=\linewidth]{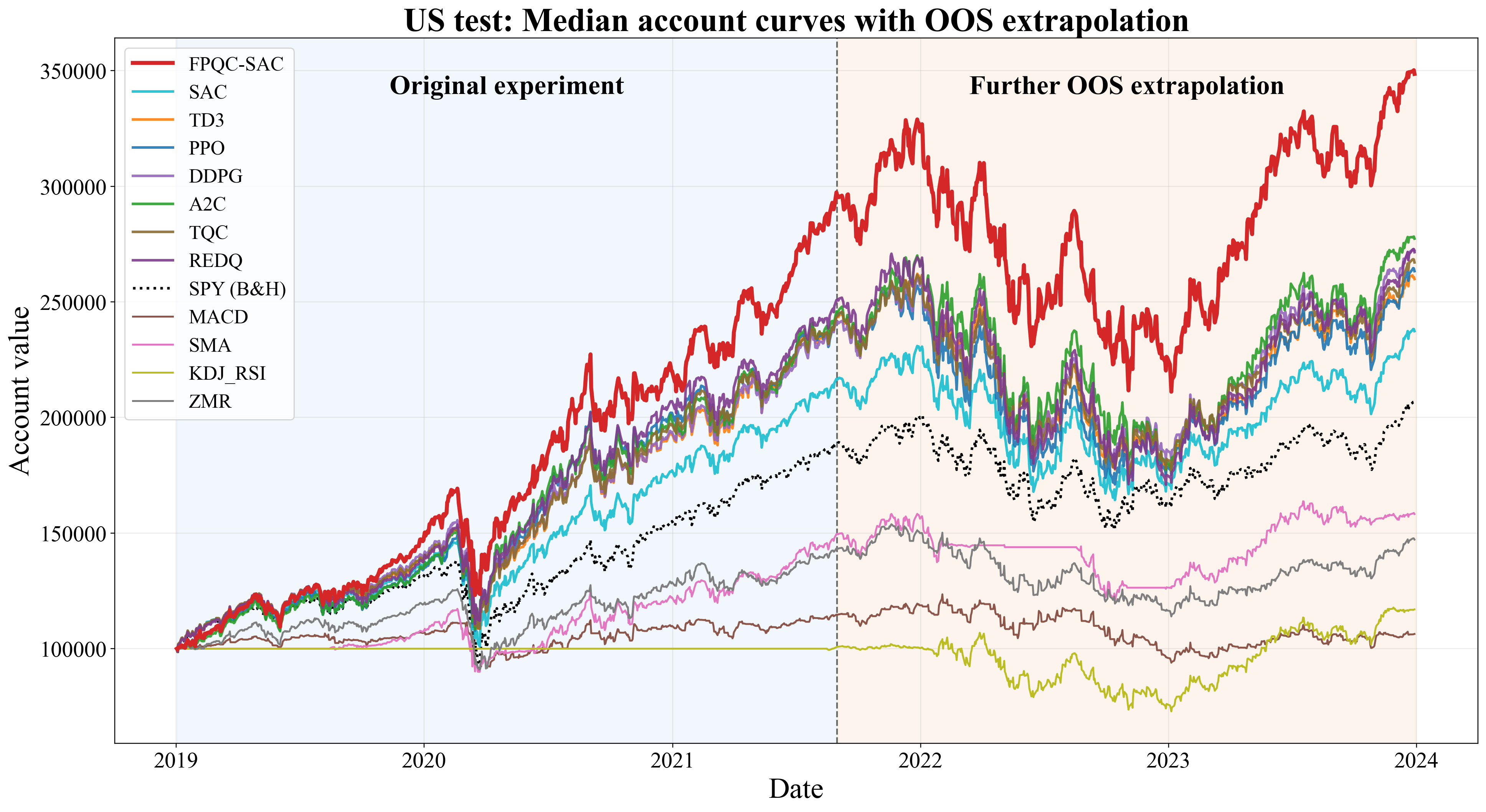}
  \caption{Extended out-of-sample median account value trajectories from 2019 to 2023. The shaded area represents the further extrapolation period beyond the original experiment. FPQC-SAC (red) maintains its relative rank and overall trajectory without structural breakdown despite the lack of retraining over the extended horizon.}
  \label{fig:waituiceshi3}
\end{figure}

Figure~\ref{fig:waituiceshi3} illustrates the median account value trajectories across this extended horizon. As expected, the systemic macroeconomic shifts in 2022 induce drawdowns across all evaluated models. However, FPQC-SAC consistently maintains its relative positioning above the DRL baselines. The policy does not exhibit catastrophic forgetting or structural breakdown over the prolonged out-of-sample period, remaining resilient even when the market environment deviates from the 2013--2018 training distribution.
\section{Statistical Significance Analysis}
\label{sec:significance}

To rigorously evaluate whether the observed performance improvements are statistically reliable, we conduct Welch's t-tests comparing FPQC-SAC against SAC across 20 independent random seeds. The results are detailed in the following table:

\begin{table}[htbp]
  \centering
  \caption{Welch's t-test significance analysis comparing FPQC-SAC against SAC.}
  \small
  \setlength{\tabcolsep}{4pt}
  \resizebox{\textwidth}{!}{
  \begin{tabular}{llccc}
    \toprule
    \textbf{Dataset} & \textbf{Metric} & \textbf{Welch's p-value} & \textbf{Significance Level} & \textbf{Statistical Confidence} \\
    \midrule
    Mainstream Tech & Cumulative Return & 4.14 $\times 10^{-10}$ & *** & >99.99\% \\
    \& Market-Index & Annual Return & 1.11 $\times 10^{-10}$ & *** & >99.99\% \\
    & Sharpe Ratio & 3.74 $\times 10^{-7}$ & *** & >99.99\% \\
    & Sortino Ratio & 3.87 $\times 10^{-9}$ & *** & >99.99\% \\
    & Calmar Ratio & 2.96 $\times 10^{-9}$ & *** & >99.99\% \\
    \midrule
    Defensive Blue-chip & Cumulative Return & 1.87 $\times 10^{-2}$ & * & >98.13\% \\
    (Sub-experiment 1) & Annual Return & 2.06 $\times 10^{-2}$ & * & >97.94\% \\
    & Sharpe Ratio & 4.90 $\times 10^{-2}$ & * & >95.10\% \\
    & Sortino Ratio & 2.83 $\times 10^{-2}$ & * & >97.17\% \\
    & Calmar Ratio & 8.72 $\times 10^{-3}$ & ** & >99.13\% \\
    \midrule
    High-Volatility Growth & Cumulative Return & 4.08 $\times 10^{-2}$ & * & >95.92\% \\
    Portfolio (Sub-experiment 2) & Annual Return & 3.52 $\times 10^{-2}$ & * & >96.48\% \\
    & Sharpe Ratio & 4.48 $\times 10^{-2}$ & * & >95.52\% \\
    & Sortino Ratio & 3.63 $\times 10^{-2}$ & * & >96.37\% \\
    & Calmar Ratio & 3.55 $\times 10^{-2}$ & * & >96.45\% \\
    \bottomrule
  \end{tabular}
  }
\end{table}

\paragraph{Strong statistical significance in volatile regime (Main Experiment).}
On the Mainstream Tech \& Market-Index Portfolio (QQQ, SPY, MSFT, GOOGL, AMZN, AAPL), which represents a technology-heavy and high-volatility regime, FPQC-SAC achieves highly statistically significant improvements across all evaluation metrics. All p-values are well below $10^{-6}$, corresponding to confidence levels exceeding 99.99\% (***). This provides strong evidence that the observed performance gains are remarkably robust and not attributable to random variation, validating the effectiveness of FPQC-SAC in low-SNR environments with strong stochastic perturbations.

\paragraph{Consistent significance in stable regime (Sub-experiment 1).}
In the Defensive Blue-chip Portfolio (Sub-experiment 1), FPQC-SAC also successfully reaches statistical significance across all metrics. The Calmar Ratio shows strong significance ($p < 0.01$), while the Cumulative Return, Annual Return, Sharpe Ratio, and Sortino Ratio all achieve conventional significance thresholds ($p < 0.05$). This is a noteworthy achievement, as stable blue-chip environments typically compress the performance margin between competing methods.

\paragraph{Resilience in extreme volatility (Sub-experiment 2).}
In the High-Volatility Growth Portfolio (Sub-experiment 2), characterized by extreme stochastic noise that typically destabilizes traditional continuous control architectures, FPQC-SAC maintains statistical significance across all evaluation metrics ($p < 0.05$). This confirms that the PQC bottleneck successfully preserves stability and performance even when market variance is severely elevated.

\paragraph{Interpretation and Implications.}
These findings highlight FPQC-SAC's universal robustness across diverse market conditions. In stable, low-variance markets where algorithmic edges are notoriously difficult to prove against static strategies, FPQC-SAC delivers systematically reliable gains. Conversely, in highly volatile regimes where dynamic decision-making is essential but prone to catastrophic variance amplification, the model effectively anchors policy updates. Ultimately, this significance analysis confirms that FPQC-SAC provides a substantial, directionally consistent, and statistically reliable advantage regardless of the underlying market regime.
\section{Evaluation Metrics}
\label{appendix:evaluation_metrics}

To comprehensively assess the trading performance and risk management capabilities of the learned policies, we employ five standard financial evaluation metrics. 

Before defining the metrics, we first formally define the portfolio net return. Let $V_t$ denote the total portfolio value at time step $t$. At each step, the agent executes trading actions that incur a transaction fee rate of $c = 0.01\%$ ($0.0001$) for both buying and selling. The net daily return $R_t$ at step $t$ naturally incorporates this transaction cost and is given by:
\begin{equation}
\label{app:eq:portfolio_net_return}
    R_t = \frac{V_t - V_{t-1}}{V_{t-1}}
\end{equation}
where $V_t$ is the net portfolio value after deducting the transaction fees associated with the generated action $a_t$. Let $T$ be the total number of trading days in the out-of-sample period, and $R_f$ be the daily risk-free return rate (set to zero in our experiments for simplicity). The metrics are defined as follows:

\paragraph{1. Cumulative Return (CR\%)}
Cumulative Return measures the overall percentage increase of the portfolio value from the beginning to the end of the trading period.
\begin{equation}
\label{app:eq:cumulative_return}
    \text{CR} = \frac{V_T - V_0}{V_0} \times 100\%
\end{equation}

\paragraph{2. Annualized Return (AR\%)}
Annualized Return represents the geometric average amount of money earned by the investment each year over a given time period. Assuming 252 trading days in a year, it is calculated as:
\begin{equation}
\label{app:eq:annualized_return}
    \text{AR} = \left( \left(1 + \text{CR}\right)^{\frac{252}{T}} - 1 \right) \times 100\%
\end{equation}

\paragraph{3. Sharpe Ratio (SR)}
The Sharpe Ratio evaluates the risk-adjusted return of the portfolio by penalizing excessive volatility. It is the ratio of the annualized expected excess return to the annualized standard deviation of returns:
\begin{equation}
\label{app:eq:sharpe_ratio}
    \text{SR} = \frac{\sqrt{252} \cdot \mathbb{E}[R_t - R_f]}{\sigma(R_t)}
\end{equation}
where $\mathbb{E}[R_t - R_f]$ is the mean of the daily excess returns, and $\sigma(R_t)$ is the standard deviation of the daily returns.

\paragraph{4. Sortino Ratio (Sortino)}
The Sortino Ratio is a variation of the Sharpe ratio that differentiates harmful volatility from total overall volatility by using the asset's downside deviation. It only penalizes negative returns:
\begin{equation}
\label{app:eq:sortino_ratio}
    \text{Sortino} = \frac{\sqrt{252} \cdot \mathbb{E}[R_t - R_f]}{\sigma_d}
\end{equation}
where $\sigma_d = \sqrt{\mathbb{E}[\min(0, R_t - R_f)^2]}$ is the target downside deviation.

\paragraph{5. Calmar Ratio (Calmar)}
The Calmar Ratio measures the annualized return relative to the Maximum Drawdown (MDD), serving as an indicator of return relative to tail-end downside risk. The MDD represents the largest peak-to-trough drop in the portfolio value:
\begin{equation}
\label{app:eq:max_drawdown}
    \text{MDD} = \max_{\tau \in (0, T)} \left( \frac{\max_{t \in (0, \tau)} V_t - V_\tau}{\max_{t \in (0, \tau)} V_t} \right)
\end{equation}
\begin{equation}
\label{app:eq:calmar_ratio}
    \text{Calmar} = \frac{\text{AR}}{\text{MDD}}
\end{equation}

\section{Limitations and Future Work}
\label{appendix:limitations_future}

While FPQC-SAC demonstrates substantial performance leaps in low-SNR financial environments, we outline the theoretical and computational boundaries of the current framework, which naturally motivate our future research trajectory.

\paragraph{Theoretical Boundaries and Degenerate Market Regimes.} 
The noise-attenuation mechanisms of FPQC-SAC are fundamentally anchored in the assumption that market dynamics exhibit It\^o-diffusion characteristics (as defined in Assumption~\ref{ass:ito}). Specifically, it excels at separating high-variance Brownian perturbations from weak but existing deterministic drift (predictive signals). Consequently, the efficacy of this denoising paradigm may be intrinsically bounded in fully degenerate market regimes---such as assets driven purely by macroscopic behavioral sentiment, meme-stock frenzies, or zero-drift microstructure noise. In such scenarios, the lack of an underlying structural signal renders any signal-preserving operation mathematically less meaningful, as the market itself behaves as pure noise. We view this not as an algorithmic deficiency, but rather as a rigorous reflection of the model's theoretical specificity: FPQC-SAC is designed to extract weak signals where they exist, rather than hallucinate signals where there are none.

\paragraph{Computational Footprint and Scalability.} 
In this study, our empirical validation focuses on highly representative, compact asset portfolios. This design choice is partially motivated by the inherent physical constraints of classical hardware: simulating the entanglement and unitary evolution of high-dimensional quantum states entails an exponentially growing computational footprint. Consequently, extending the current classical-simulation-based FPQC-SAC to massively scaled combinatorial portfolios (e.g., thousands of concurrent assets) presents practical computational bottlenecks. However, our strategic use of the PQC as a \textit{compact bottleneck} elegantly turns this limitation into an advantage, enabling efficient deployment on current classical GPUs without compromising the architectural integrity of the quantum representation.

\paragraph{Future Work: Transitioning to Physical Quantum Hardware.}
The computational bottleneck of simulating quantum mechanics on classical machines is precisely what quantum hardware is designed to resolve. As an exciting and natural progression, our immediate future work aims to deploy the FPQC-SAC architecture onto real-world physical Quantum Processing Units (QPUs). By transitioning from classical simulators to actual quantum hardware, we anticipate bypassing the exponential simulation overhead. This will allow us to scale the qubit register to accommodate massive-scale portfolio management tasks, fully unlocking the intrinsic parallel processing power and natural entanglement dynamics of genuine quantum systems in real-time financial trading.

\section{Broader Impact}
This paper presents work whose goal is to advance the intersection of quantitative finance and artificial intelligence. To ensure reproducibility, we have open-sourced the official implementation at \url{https://github.com/ZeyuLIU-UST/FPQC-SAC-main}. While our primary focus is on algorithmic stability and methodological improvements in reinforcement learning, we acknowledge that deploying automated trading systems in real-world financial markets carries inherent economic implications and market risks, which necessitate careful risk management in practical applications.

\end{document}